%% file: main.tex
\documentclass{scrartcl}

\usepackage[T1]{fontenc}
\usepackage{amsthm}
\usepackage{verbatim,url}
\usepackage{booktabs}
\usepackage{amsmath}
\usepackage{amssymb}
\usepackage{amstext}
\usepackage{amsfonts}
\usepackage{epsfig}
\usepackage{subfig}
\usepackage{graphicx}
\usepackage{grffile}
\usepackage[mathscr]{eucal}
\usepackage{latexsym}
\usepackage{bbm}
\usepackage{amsopn}
\usepackage{float}
\usepackage{caption}
\usepackage{balance}
\usepackage{color}
\usepackage{cite}
\usepackage{algorithmic}
\usepackage{algorithm}
\usepackage{xspace}
\newcommand{\eat}[1]{}
\usepackage{bbold,bm}
\usepackage{microtype}

\def\BibTeX{{\rm B\kern-.05em{\sc i\kern-.025em b}\kern-.08em
    T\kern-.1667em\lower.7ex\hbox{E}\kern-.125emX}}

\newtheorem{lemma}{Lemma}
\newtheorem{theorem}{Theorem}

\newcommand{\linucb}{\textsc{LinUCB}\xspace}
\newcommand{\linrel}{\textsc{LinREL}\xspace}
\newcommand{\ts}{\textsc{ThompsonSampling}\xspace}

\newcommand{\fullversion}[2]{#2} 

\def\0{{\mathbf 0}}
\def\1{{\mathbf 1}}

\DeclareMathOperator*{\argmax}{arg\,max}
\DeclareMathOperator*{\argmin}{arg\,min}

\newcommand{\diag}{\mathop{\mathtt{diag}}}

\newcommand{\trace}{\mathop{\mathtt{trace}}}
\newcommand{\vect}{\mathop{\mathtt{vec}}}

\newcommand{\ub}{\mathbf{u}}
\newcommand{\vb}{\mathbf{v}}
\newcommand{\zb}{\mathbf{z}}

\newcommand{\ubb}{\mathbb{u}}
\newcommand{\vbb}{\mathbb{v}}
\newcommand{\zbb}{\mathbb{z}}
\newcommand{\reals}{\mathbb{R}}
\newcommand{\naturals}{\mathbb{N}}

\newcommand{\x}{\mathbf{x}}

\eat{

}

\newcommand{\squishlist}{
 \begin{list}{$\bullet$}
  {  \setlength{\itemsep}{0pt}
     \setlength{\parsep}{3pt}
     \setlength{\topsep}{3pt}
     \setlength{\partopsep}{0pt}
     \setlength{\leftmargin}{2em}
     \setlength{\labelwidth}{1.5em}
     \setlength{\labelsep}{0.5em}
} }
\newcommand{\squishlisttight}{
 \begin{list}{$\bullet$}
  { \setlength{\itemsep}{0pt}
    \setlength{\parsep}{0pt}
    \setlength{\topsep}{0pt}
    \setlength{\partopsep}{0pt}
    \setlength{\leftmargin}{2em}
    \setlength{\labelwidth}{1.5em}
    \setlength{\labelsep}{0.5em}
} }

\newcommand{\squishdesc}{
 \begin{list}{}
  {  \setlength{\itemsep}{0pt}
     \setlength{\parsep}{3pt}
     \setlength{\topsep}{3pt}
     \setlength{\partopsep}{0pt}
     \setlength{\leftmargin}{1em}
     \setlength{\labelwidth}{1.5em}
     \setlength{\labelsep}{0.5em}
} }

\newcommand{\squishend}{
  \end{list}
}

\title{Bandits Under the Influence}
\subtitle{Extended Version}

\author{Silviu Maniu\\
  LRI, CNRS\\
    Universit\'e Paris-Saclay\\
    Orsay, France\\
    silviu.maniu@lri.fr
  \and
  Stratis Ioannidis\\
  Electrical and Computer Engineering\\
    Northeastern University\\
    Boston, MA, USA\\
    ioannidis@ece.neu.edu
  \and
  Bogdan Cautis\\
  LRI, CNRS\\
    Universit\'e Paris-Saclay\\
    Orsay, France\\
    bogdan.cautis@lri.fr
}

\date{}

\begin{document}

\maketitle

\begin{abstract}
  Recommender systems should adapt to user interests as the latter evolve. A prevalent cause for the evolution of user interests is the influence of their social circle. In general, when the interests are not known,  online algorithms that explore the recommendation space while also exploiting observed preferences are preferable.
  We present online recommendation algorithms rooted in the linear multi-armed bandit literature. Our bandit algorithms are tailored precisely to recommendation scenarios where user interests evolve under social influence.  In particular, we show that our adaptations of the classic \textsc{LinREL} and \ts algorithms maintain the same asymptotic regret bounds as in the non-social case. We validate our approach experimentally using both synthetic and real datasets.
\end{abstract}

\section{Introduction}


Recommender systems 
can benefit significantly from sequential learning techniques, such as multi-armed bandit algorithms,  when user interests are a priori unknown,  hardly generalizable, or  highly dynamic. Such conditions arise in news recommendation scenarios, where the turnover of items is simply too high to enable a reasonable application of traditional recommendation algorithms, or in cold-start scenarios, i.e., when addressing a new or ever-changing user base.  Online recommendation algorithms (re)learn preferences over time and continuously, striking a balance between exploiting  popular recommendation options and exploring new ones, that may improve overall  user satisfaction.  

One  prevalent reason for the continuous evolution of user interests, calling for such online learning approaches for recommendation, is \textit{social influence}, under which connected users converge to similar interests. While realistic models for social influence remain an only partially understood area, 
the presence of influence in social media -- by either global and local mechanisms -- has been  extensively studied formally \cite{kempe03, DBLP:conf/kdd/DomingosR01} and verified experimentally ~\cite{bakshy2011everyone,cheng2014cascades}.


Motivated by the above observations, we study  an online recommender system \emph{that learns user interests as they evolve under social influence}. 
We consider a scenario in which the initial interests of users are unknown, but the effects of the social network on their evolution are understood and modelled by a probabilistic social graph -- akin to the independent cascade model \cite{lei2015online}. 
More precisely, the recommender follows the combined objective of maximizing rewards (i.e., user ratings) over a finite horizon, while simultaneously discovering user interests. The latter however are subject to a drift caused by social influence.

This setting gives rise to several challenges not present in  classic recommender systems. The first is the usual multi-armed bandit (MAB) challenge of exploration, in discovering and tracking user interests, vs.~exploitation, via the recommendation of pertinent content. The second (and most crucial) challenge, differentiating us from classic MAB and recommendation literature alike, is that \emph{recommendations become coupled via social influence}. As a result, the status quo of targeting recommendations on individual users separately is suboptimal. Instead, social influence implies that a \emph{global} recommendation strategy needs to be optimized across users. This leads to a combinatorial explosion of the space of possible recommendation strategies, as the latter grow exponentially in the number of users.

Our main contribution is to address these two challenges in a comprehensive fashion. In particular:
\squishlist
    \item We show that the \textsc{LinREL} \cite{auer2003confidence,dani2008stochastic}  and \textsc{ThompsonSampling}  \cite{agrawal2014thompson} bandit algorithms are a \emph{natural fit to our setting}. We provide \emph{regret} bounds for both methods, taking into account the effect of social influence. 
    \item Crucially, we  establish that both algorithms are \emph{tractable}: despite the exponential size of possible global recommendation strategies (i.e., arms) in our setting, we derive polynomial-time algorithms for arm selection under both the \linrel and the \ts algorithms.
    \item We also consider \textsc{LinUCB}, another popular linear bandit strategy \cite{chu2011contextual,cesabianchi2013gang, wu2016contextual, li2017provably}; the corresponding arm selection process turns out to be intractable, but a solution can be approximated within a constant in polynomial time; unfortunately, this  implies that the resulting algorithm comes with no regret guarantees.
\squishend
Importantly, to the best of our knowledge, we are the first to analyze and compare \linrel, \ts, and \linucb, both from a tractability and a regret perspective, in the presence of social influence. From a technical standpoint, our analysis requires both revisiting regret bounds under a dynamic influence setting, but also tackling the exponential size of the recommendation strategy space. We accomplish the latter through the reduction of the arm selection process to an optimization with a linear objective, which becomes separable across users.

\fullversion{We omit proofs and experimental details from this short paper; both can be found in the extended version \cite{fullversion}.}{
The remainder of the paper is organized as follows. We discuss related work in Sec.~\ref{sec:related} and formally state our problem in Sec.~\ref{sec:problemformulation}. We present our bandit algorithms, along with guarantees, in Sec.~\ref{sec:linrel}-\ref{sec:linucb}. Our experimental evaluations are in Sec.~\ref{sec:exp}. We conclude in Sec.~\ref{sec:conclusions}.
}


\fullversion{}{
\section{Related Work}\label{sec:related}
  
There is an extremely rich literature on recommender systems 
(see, e.g., the recent survey by Ricci~\cite{DBLP:reference/snam/Ricci18}). The temporal dynamics of user interests have been accounted for in many  recent works in this area \cite{DBLP:journals/cacm/Koren10,DBLP:conf/wsdm/SarmaGPZ12}.  Approaches for recommendation in social media, including aspects such as social influence, have also been considered extensively in recent years (see the survey by Eirinaki et al.~\cite{Eirinaki18}), though not in the formal bandit setting we study here.   

Online recommendation algorithms, like the one by El-Arini et al.~\cite{El-arini09turningdown}, encompass approaches that learn the underlying parameters while running recommendation campaigns. In order to guide users through the flood of information in social media, an  online  learning  framework should quickly learn  user preferences from limited feedback, while minimizing the incurred penalty (regret). This can be cast as an online learning-to-rank problem  \cite{DBLP:conf/recsys/OdijkS17}.  Multi-armed bandits have  been used  in recommendation scenarios where the model must be learned and updated continuously during recommendation campaigns \cite{DBLP:journals/corr/abs-1904-07272}. Contextual linear bandits also have a long history in online learning \cite{auer2003confidence,dani2008stochastic,chu2011contextual,abbasi2011improved,agrawal2014thompson,li2017provably}; generalized linear bandits have also been extensively studied~\cite{li2017provably, abbasi2011improved, DBLP:conf/icml/AgrawalG13}. We depart from the classic setting in that ``arms'' (recommendations) are subject to the  drift induced by social influence (see Lemma~\ref{lem:quadform}). 

Contextual information may also come in the form of social influence/social ties, and this is one of the main assumptions of our work. Motivated by viral marketing in social media, influence estimation and influence
maximization have become important research problems ever since the seminal works of Kempe et al. \cite{kempe03}  and Domingos and Richardson \cite{DBLP:conf/kdd/DomingosR01}. These problems address respectively the challenges of maximizing the expected spread in a social graph  and estimating user influence from past observations (information cascades), under a certain diffusion model. In particular, two stochastic,
discrete-time diffusion models, Linear Threshold (LT) and Independent Cascade (IC), have been used in the bulk of the literature on these topics, precisely because the expected number of nodes reached is a sub-modular function of the seed set under these models \cite{kempe03}. The influence model we consider here is related to IC but is somewhat more complex, as user interests, which are vectors, evolve towards averages over their social neighbourhood. As the consequence of such averaging, it is also related to PageRank \cite{pagerank}, gossiping models \cite{DBLP:journals/ftnet/Shah09} and the average model in social voting systems~\cite{das2014modeling}.

The works closest in spirit to our setting assume that the rewards are a function of the social neighbourhood of the user. Bianchi at al.~\cite{cesabianchi2013gang} assume that a  user requiring a recommendation is given in each round, and that each node can have affinities with other nodes -- without any dynamic component. This affinity translates into a measure of closeness in the payoffs at each node. Their algorithm, \textsc{GobLIN}, is an adaptation of \linucb. This setting is refined by Wu et al.~\cite{wu2016contextual}, where the user payoffs are a mixture of their neighbors' payoffs.  Similarly, Li et al.~\cite{li2016collaborative} use a multi-user linear bandit formulation where users are clustered together via their static profile vectors; this allows good computational performance but it is far from our more general setting, where users can change, in time, to \emph{any} other profile. There are several major differences in our setting. First, we model an influence process; no assumption of similarity between users (e.g., in user inherent interests), beyond the one induced by the influence process, is present in our model. Second, we  extend our analysis to \linrel and \ts. Third, as another common characteristic of the above methods, in their practical implementations, similar nodes are clustered together and receive identical recommendations-- resulting, essentially, in a drastic reduction of the search space, as each cluster can be considered as one bandit.  This works in static settings; in our case, the profiles evolve in a dynamic way, so no clustering is possible in general. Moreover, we model a joint recommendation system, one in which each user receives a recommendation at each step, which leads to a combinatorial explosion in the recommendation space. We note that the non-stationary setting is also studied by Russac et al.~\cite{russac2019weighted}, where a drift on the profile of the user is analyzed, without however any social influence aspect and, hence, any coupling effects.

Finally, our work generalizes and extends the work by Lu et al.~\cite{lu2014optimal}: their setting is more restrictive in that  (a) inherent user profiles are known and (b) the system is in steady state, though their dynamics include, beyond social influence, attraction and aversion phenomena. We consider instead an online setting, where unknown  user interests evolve under social influence. There are, however, some limitations of our model: it does not model user profiles evolving with the recommendations~\cite{gu2017coevolution} or other product effects, such as complementarity or relationships between items~\cite{zhao2017improving}.
}

\input{problem}

\input{experiments}

\fullversion{}{
\section{Conclusions}\label{sec:conclusions}

We presented online recommendation algorithms based on linear multi-armed bandit algorithms, for recommendation scenarios where user profiles evolve due to social influence. We showed that adapting \linrel and \ts to this setting leads to regret bounds that are similar to their original, static, settings. Importantly, we provide tractable cases for both algorithms; they are implemented and validated experimentally on both synthetic and real-world datasets.
}

\input{acknowledgements}

\bibliographystyle{abbrv}
	\bibliography{bibliography}


\fullversion{}{\appendix

\input{proof}

\input{sdp}
}

\end{document}

%% file: problem.tex
\fullversion{}{
\begin{table}
\vspace{3mm}
\begin{scriptsize}
\begin{tabular}{p{0.07\columnwidth}p{0.33\columnwidth}|p{.07\columnwidth}p{0.33\columnwidth}}
\hline
$n$ & number of users & $d$ & feature dimension\\
$[n]$& Set $\{1,\ldots,n\}$ & $\mathcal{B}$ & Recommendation set \\
$\ub,\vb,\zb$ & vectors in $\reals^d$ & $\ubb,\vbb,\zbb$ & vectors in $\reals^{nd}$ \\
$\ub_i(t)$ & User $i$'s profile in at time $t$ & $\vb_i(t)$ & recommendation to $i$ at time $t$\\
$U(t)$ & Matrix of user profiles & $V(t)$ &Matrix of recommended items\\
$r_i(t)$ & Rating by $i$ & $\alpha$ & Inherent probability\\
$P_{ij}$& Influence from $j$ to $i$ & $P$ & influence matrix\\
$\ub_i^0$ & inherent profile & $U^0$ &matrix of inherent profiles\\
$\ubb_0$ & $\vect(U^0)$ & $\vbb$ & $\vect(V)$ \\
$\langle \cdot,\cdot\rangle$ & Frobenius matrix inner product &  $\otimes$ & Kronecker product\\
$A(t) $ & $n\times n$ matrix given by \eqref{eq:social_update} & $L(t)$ & $nd\times nd$ matrix given by \eqref{L}\\
$R(T)$ & Regret at time $T$ & $\bar{r}$ & total expected reward \\
$\mathbf{r}$ & vector of expected rewards & $X$ & $n\times nd$ matrix given by \eqref{xarms}\\
\hline
\end{tabular}
\end{scriptsize}
\caption{Notation Summary.}\vspace*{-6mm}
\end{table}
}

\section{Problem Formulation}\label{sec:problemformulation}

\fullversion{}{We consider users in a social network that receive recommendations. Following \cite{lu2014optimal}, a user's interests are dynamic and are affected by  her neighbors. In particular, user reactions to recommendations are driven by two components, namely, (a) \emph{inherent behavior},  capturing a predisposition users may have towards particular topics or genres, and (b) \emph{social-influence}, capturing the effect of a user's social circle. Contrary to \cite{lu2014optimal}, the recommender system needs to not only make good recommendations, but also discover user interests \emph{while  accounting for social influence}.}

\subsection{Recommendations}
Formally, we consider $n$ users, all receiving suggestions from a recommender at discrete time steps $t\in \naturals$. Each recommended item is represented by a $d$-dimensional \emph{item profile} $\vb\in\reals^d$, capturing this item's features; we denote by $\mathcal{B}\subseteq \reals^d$ the set of available items, i.e., the recommender's \emph{catalog}. We denote by $[n]\equiv\{1,2,\ldots,n\}$ the set of all users. At each time step $t\in \naturals$,
the recommender suggests (possibly different) items from $\mathcal{B} $ to each user  $i\in [n]$. Each  $i$  responds by revealing a rating $r_i(t)\in\reals$, indicating her preference towards the item recommended to her.

We assume that ratings are generated according to the following random process. At each $t\in \naturals$,  users $i\in[n]$ have \emph{user profiles}
	represented by $d$-dimensional vectors $\ub_i(t) \in \reals^d$.
Then, if $\vb_i(t)\in \mathcal{B}$ is the profile of the item recommended to $i$ at time $t$, 
 ratings satisfy: 
\begin{align}r_i(t)=\langle\ub_i(t),\vb_i(t)\rangle +\varepsilon, \label{factor}\end{align}
	where $\varepsilon$ is zero mean, finite variance noise (i.i.d.~across users and timeslots). \fullversion{}{This ``bi-linear'' model is classic: it is the cornerstone of matrix factorization approaches \cite{bell2009matrix}. Contrary to  matrix factorization literature, however, user profiles $\ub_i(t)$, $i\in [n]$,  change through time. We describe their evolution below.}

\subsection{User interest evolution}\label{sec:interestevol}

  Following \cite{lu2014optimal}, we assume that profiles \emph{evolve} according to the following dynamics. Each user is associated with an inherent profile, capturing her personal interests. At each timeslot, she chooses with some probability to either use her inherent profile, or she chooses to use a profile that is the result of the influence of her neighborhood. Formally, at each time $t\in \naturals$, user profiles evolve according to:
  \begin{align}\ub_i(t) = \alpha \ub_i^0 + (1-\alpha)\textstyle \sum_{j\in [n]}P_{i,j} \ub_{j}(t-1),~i\in [n],\label{evol} \end{align}
	  where (a) $\ub_i^0\in \reals^d$ is user $i$'s inherent (static) profile, (b)  $\alpha\in[0,1]$ captures the probability that users act based on their inherent profiles, and (c)  $P_{ij}\in [0,1]$, $i,j\in [n]$, where $\sum_{j}P_{ij}=1$, capture the probability user $i$ is influenced by the profile of user $j$. This setting can allow different values for $\alpha$ for each user; this does not change the analysis in the following.
	 
	 The probability $P_{ij}\in[0,1]$ captures the influence that user $j$ has on user $i$. Note that users $j$ for which $P_{ij}=0$ (i.e., outside $i$'s social circle) have no influence on $i$. Moreover,  the set of pairs $(i,j)$ s.t. $P_{ij}\ne 0$, defines the social network among users.  We denote by $P\in [0,1]^{n\times n}$  the stochastic matrix with elements $P_{ij}$, $i,j\in [n]$; we assume that $P$ is ergodic (i.e., irreducible and aperiodic) \cite{gallager1996discrete}, a reasonable setting -- even if the original social influence graph is not, one can easily make it so by adding a small probability between all pairs in the graph.
	  Then, for  
   $U(t)=[\ub_i(t)]_{i\in[n]}\in \reals^{n\times d}$ the $n\times d$ matrix consisting of the profiles of all users at time $t$, \eqref{evol} can we written in matrix form as:
\begin{align}U(t) = \alpha U^0 + (1-\alpha) P U(t-1),\label{snalone}\end{align}
	where matrix $U^0=[\ub_i^0]_{i\in[n]}\in \reals^{n\times d}$ comprises  inherent profiles.  
%
%
%
%

\subsection{Bandit setting}

We consider a bandit setting, in which the recommender does not know the inherent profiles, but would nevertheless wish to  maximize rewards \eqref{factor} over a finite horizon. In particular, we assume that the recommender (a) knows the probabilities $\alpha$ and $P$, capturing the dynamics of the interest evolution and (b) observes the history of responses by users in previous timeslots $t=1,\ldots,t-1$. Based on this history, the recommender suggests items $\vb_i(t)\in \mathcal{B}$,  $i\in [n]$, in order to minimize the  \emph{aggregate regret}:
\begin{align}\label{eq:regret}R(T)=\textstyle \sum_{t=1}^T\sum_{i=1}^n \langle \ub_i(t),\vb_i^*(t)\rangle - \langle \ub_i(t), \vb_i(t)\rangle, \end{align}
	where $\vb_i^*(t)$, $i\in [n]$, $t\in [T]$, are recommendations made by an optimal  strategy that knows  vectors $\ub^0_i$, $i\in [n]$ (see  Eq.~\eqref{eq:regret2}  for a closed form characterization).

 Recommender suggestions $\mathbf{v}_i(t)$, $i\in [n]$ are selected from a set $\mathcal{B}\subseteq \reals^d$. We consider two possibilities: 
\begin{itemize}
\item $\mathcal{B}$ is a finite subset of $\reals^d$, i.e., it is  a ``catalog'' of possible recommendations.
\item $\mathcal{B}$ is an arbitrary convex subset of $\reals^d$, e.g., the unit ball $\mathbb{B}\equiv \{\vb\in \reals^d:\|\vb\|_2 \leq  1\}$.
\end{itemize}
\fullversion{}{We note that the recommendation problem associated with minimizing the aggregate regret \eqref{eq:regret} poses multiple challenges. The first is the standard MAB challenge of exploration vs.~exploitation w.r.t.~ discovering inherent profiles. The second is that the presence of social influence \emph{couples} recommendations across users. In particular, \emph{we cannot treat maximizing the social welfare (the sum of rewards) as $n$ individual/personalized recommendation problems, as is standard practice}. This significantly increases the difficulty of finding a regret-minimizing recommendation strategy. For example, the finite catalog case, the recommender needs to  consider all $|\mathcal{B}|^n$ possible recommendations, and the space of possible joint suggestions is exponential in $n$. The fact that the catalog size $|\mathcal{B}|$ may be large exacerbates the problem.}

\subsection{Relationship to linear bandits}
The aggregate expected reward is, at any time $t$, a linear function of the inherent user profiles $U^0$. This motivates our exploration of \linrel, Thompson sampling, and \linucb as candidate online algorithms with bounded regret.  
To see this, let  $V(t)=[\vb_i(t)]_{i\in[n]}\in \mathbb{R}^{n\times d}$, be the matrix comprising the recommendations made at time $t$ in each row. Then, the following lemma holds:
\begin{lemma}\label{lem:lin}
	The total expected reward $\bar{r}(t)$ at time $t$ under recommendations $V(t)$ is given by:
\begin{align}
\bar{r}(t) = \langle U^0, A(t)^\top V(t)\rangle, \label{rew}
\end{align}
where $\langle A, B\rangle = \mathop{\trace}(AB^\top),$ is the Frobenius inner product, 
and 
\begin{align}
A(t) = \alpha \sum_{k=0}^t \left(\left(1-\alpha\right) P \right)^k \in \mathbb{R}^{n\times n}.\label{eq:social_update}
\end{align}
\end{lemma}
\fullversion{}{
\begin{proof}
 By induction on \eqref{snalone}, one can show that:
\begin{align}
U(t) &= \alpha U^0 + \alpha(1-\alpha) P U^0 + \ldots + ((1-\alpha) P)^t \alpha U^0\nonumber\\
& = A(t) U^0
\end{align}
	for $t\geq 1$, where $A(t)$ is given by \eqref{eq:social_update}.		
Hence, the total expected reward is given by:
\begin{align}
\bar{r}(t) &= \langle A(t) U^0, V(t)\rangle = \trace(A(t)U^0 V^{\top}(t))\nonumber\\
& =\trace( U^0 V^\top (t) A(t))
=\langle U^0, A(t)^\top V(t)\rangle. \end{align}
where $\langle \cdot,\cdot \rangle$ is the Frobenius inner product. 
\end{proof}
}
Lemma~\ref{lem:lin} gives a clearer indication that the reward is indeed a linear function of the unknown inherrent profiles $U^0$. To ease exposition, but also to be consistent with existing bandit literature, we  vectorize the representation of recommendations and user profiles.
Given a $U^0,V\in \reals^{n\times d}$, we define their row-wise vector representations as:
\begin{align*}
\ubb_0&\equiv \mathtt{vec}(U^0) =[(\ub_1^0)^\top,(\ub_2^0)^\top,\ldots,(\ub_n^0)^\top]\in \reals^{nd}\\
\vbb &\equiv \vect(V) = [\mathbf{v}_1^\top,\mathbf{v}_2^\top, \ldots,\mathbf{v_n}^\top]\in \mathcal{B}^{n} \subseteq \reals^{nd},
\end{align*}
i.e., $\ubb_0$, $\vbb$ are  $nd$-dimensional vectors resulting from representing $U^0$, $V$ row-wise. 
We can then directly describe the expected reward at time $t$ as a quadratic form involving $\ubb_0,\vbb(t)$:

\begin{lemma}\label{lem:quadform}
The total expected reward at time $t$ under recommendations $V=V(t)$ is then given by:
\begin{align}
	\bar{r}(t) =  \ubb_0^\top L(t) \vbb \label{eq:bilincomp}
\end{align}
where $\ubb_0\equiv \vect(U^0)$, $\vbb \equiv \vect(V)$, and 
 \begin{align}
	 L(t)\equiv  A(t)^\top \otimes I_d\in \reals^{nd\times nd}\label{L}
 \end{align}
 is  the Kronecker product of $A(t)^\top$ with the identity matrix $I_d\in \reals^{d\times d}$.
\end{lemma}
%
\fullversion{}{\begin{proof}The  vector  $\bar{\mathbf{r}}(t)\in \reals^n$ of expected  rewards is:
\begin{align}
\bar{r}(t) &= \diag( A(t) U^0 V^\top(t) ) = X(V(t),A(t)) \ubb_0,
\end{align}
where $X(V(t),A(t))\in \reals^{n\times nd}$ is the ``context'' (in contextual linear bandits sense) given by:
\begin{align}X(V,A) \equiv \left[ \begin{matrix}
a_{11} \mathbf{v}_1^\top &  a_{12} \mathbf{v}_1^\top & \ldots &  a_{1n} \mathbf{v}_1^\top\\
a_{21} \mathbf{v}_2^\top &  a_{22} \mathbf{v}_2^\top & \ldots &  a_{2n} \mathbf{v}_2^\top\\
& \vdots& \ddots & \vdots\\
a_{n1} \mathbf{v}_n^\top &  a_{n2} \mathbf{v}_n^\top & \ldots &  a_{nn} \mathbf{v}_n^\top\\
  \end{matrix}
\right]\label{xarms}
\end{align}
The lemma follows by observing that: 
\begin{align*}\bar{r}(t) &= \mathbf{1}^\top X(V(t),A(t))) u_0 =\langle U^0, A^\top (t) V(t)\rangle \\&= \ubb_0^
\top \vect(A^\top(t) V) = \ubb_0^\top   L(t) \vbb.
 \end{align*}
 where $L(t)$ is given by \eqref{L}. \end{proof}}

Lemma~\ref{lem:quadform} 
casts our reward (and our problem) in a linear (a.k.a.~contextual) bandit setting. The quadratic form  \eqref{eq:bilincomp} suggests that the reward is linear in both the unknown parameters $\ubb_0$ \emph{and} the recommendations $\vbb$; in turn, the evolution of interests changes the nature of this relationship via $nd\times nd$ matrix $L(t)$; in particular, the latter fully determines the coupling between recommendations across users. Finally, Lemma~\ref{lem:quadform} allows us to rewrite the regret \eqref{eq:regret} as follows:
\begin{align}\label{eq:regret2}
R(T)&=\textstyle\sum_{t=1}^T \big(\ubb_0^\top L(t)\vbb^*(t)-\ubb_0^\top L(t)\vbb(t)\big)
\end{align}
where $\vbb^*(t)=\argmax_{\vbb\in \mathcal{B}^{(n)}}\ubb_0^\top L(t)\vbb  $ is the optimal decision at  $t$. 

\fullversion{}{Armed with this representation, we turn our attention to linear bandit algorithms (\textsc{LinREL}, \textsc{ThompsonSampling}, and \textsc{LinUCB}) and their application to our problem. 
We deviate from the standard setting  precisely due to the (drifting, time-variant) matrix $L(t)$. As a result,  any classic regret results for linear bandit algorithms need to be revisited. Beyond this, however, the exponential size of the action space poses an additional challenge, as  existing algorithms may be intractable. As we discuss below, this is indeed the case for \textsc{LinUCB}; our major contribution is to show that \textsc{LinREL} and \textsc{ThompsonSampling} can be applied to our setting while still (a)  obtaining bounded regret, and (b) remaining tractable.
}

\section{\linrel Algorithm}\label{sec:linrel}

The \linrel algorithm \cite{auer2003confidence}, also called ``confidence ball'' in \cite{dani2008stochastic}, operates under the following assumptions: arms are selected from a vector space, and the expected reward observed is an (unknown) linear function of the arm selected. To select an arm, the algorithm uses a variation of the Upper Confidence Bound (UCB) principle, in that it considers a confidence bound to an estimator for each arm's reward when selecting it. The unknown linear model is estimated via a least squares fit; the upper confidence bound  constrains the next selection over an $L_1$ or $L_2$ ellipsoid centred around the current estimate. In our case, the resulting arm selection problem is non-convex; nevertheless, as we show below, it can be solved in polynomial time for a variety of different settings.


\begin{algorithm}[!t]
  \caption{ -- \linrel  }\label{alg:linrel}
  \begin{algorithmic}[1]
    \REQUIRE{matrix $P$, parameter $\alpha$, item set $\mathcal{B}$, users $[n]$}
    \STATE{\textbf{Initialization:} play $d$ pulls for each $i\in[n]$ and observe rewards $\textbf{r}_0$}
    \STATE{$A(0)\gets\alpha I$}
    
    \FOR{$t = 1, \ldots, T$}
    \STATE{estimate $$\hat{\ubb}_0(t) = \textstyle \argmin_{\ubb\in \reals^{nd}}\sum_{\tau=1}^{t-1} \|X(V(\tau),A(\tau))\ubb - \mathbf{r}(\tau)\|_2^2$$}
    \STATE{ Recommend  $$\textstyle \vbb_t=\argmax_{\vbb\in\mathcal{B}^{(n)}}\max_{\ubb\in\mathcal{C}_t}\ubb^\top L(t)\vbb$$}\label{line:optimization}
    where $\mathcal{C}_t$ is given by \eqref{c2} or \eqref{c1}
    \STATE{observe reward vector $\mathbf{r}(t)$}
    \STATE{$A(t)\gets A(t-1)+\alpha(1-\alpha)P^t$}
    
    \ENDFOR   
  \end{algorithmic}
\end{algorithm}

\subsection{Algorithm overview}
Lemma~\ref{lem:quadform} indicates that the total expected reward in our setting is indeed a linear function of ``arms'' $\vbb\in\reals^{nd}$, parametrized by unknowns $\ubb_0\in \reals^{nd}$; nevertheless, the arms are affected by the current state of influence, as captured by  $L(t)\in \reals^{nd\times nd}.$ Applied to our setting, LinREL operates as summarized in Algorithm~\ref{alg:linrel}.
As an initialization step, for each user in $[n]$, the recommender suggests $d$ arbitrary items that  span $\mathbb{B}$. For each round, ${\ubb}_0(t)$ is estimated via least squares estimation from past observations, i.e.:
\begin{align}\label{eq:lse}\hat{\ubb}_0(t) = \argmin_{\ubb\in \reals^{nd}}\sum_{\tau=1}^{t-1} \|X(V(\tau),A(\tau))\ubb - \mathbf{r}(\tau)\|_2^2\end{align}
%
where $X\in \reals^{n\times nd}$ \fullversion{is an expanded arm matrix depending on both $V$ and $A$ (see \cite{fullversion})}{is given by \eqref{xarms}} and $\mathbf{r}(\tau)\in \reals^n$ is the vector of rewards collected (i.e., user responses observed) at time $\tau$.


At iteration $t\geqslant 1$, the recommendations $\vbb(t)$ are selected by \linrel as the solutions of the optimization 
\begin{align}\vbb(t)=\argmax_{\vbb\in\mathcal{B}^{(n)}}\max_{\ubb\in\mathcal{C}_t}\ubb^\top L(t)\vbb\label{eq:opt}\end{align}
where $\mathcal{C}_t$ is an appropriately selected ellipsoid centered at $\hat{\ubb}_0$ and capturing the uncertainty of the estimate. Two possible cases considered by \cite{dani2008stochastic,auer2003confidence} are:
\begin{subequations}\label{eq:c}
\begin{align}
    \mathcal{C}^2_t &= \left\{\ubb: \|\hat{\ubb}_0(t)-\ubb\|_{2,Z(t)}\leq \sqrt{\beta_t}\right\},\label{c2}\\
    \mathcal{C}^1_t &= \left\{\ubb: \|\hat{\ubb}_0(t)-\ubb\|_{1,Z(t)}\leq \sqrt{nd\beta_t}\right\},\label{c1}
\end{align}
\end{subequations}
where $\beta_t=\beta(t,n,d,\delta)$  is a closed-form function of $t$, $n$, $d$, and a parameter $\delta$, $Z(t)$ is the inverse of the covariance (i.e., the precision) of estimate $\hat{\ubb}_0$, given by 
\begin{align}
Z(t) 
&=\textstyle\sum_{\tau=1}^{t-1}X\big(V(\tau,A(\tau))^\top X(V(\tau,A(\tau)\big),
\end{align}
and $\|\mathbf{x}\|_{2,A} = \sqrt{\mathbf{x}^\top A\mathbf{x}},$ 
$\|\mathbf{x}\|_{1,A} = \|A^{1/2}\mathbf{x}\|_1$.


Intuitively, the optimization \eqref{eq:opt} selects a recommendation that maximizes the expected reward \emph{under a perturbed} $\ubb$: though $\mathcal{C}_t$ is centered at $\hat{\ubb}_0$, directions of high variability are favored under optimization \eqref{eq:opt}. A key challenge is that \emph{\eqref{eq:opt} is not a convex optimization problem}; in fact, when the catalog $\mathcal{B}$ is a finite set, \eqref{eq:opt} is combinatorial, and the set $\mathcal{B}^{(n)}$ grows exponentially in $n$. Nevertheless, as we discuss in Sec.~\ref{sec:c1} below, when $\mathcal{C}_t=\mathcal{C}_t^1$, we can solve \eqref{eq:opt} efficiently for the different cases of set $\mathcal{B}\subset \reals^d$  presented in Sec.~\ref{sec:interestevol}, including finite catalogs.




\subsection{Regret} Most importantly, we can show the following bound on the regret:
\begin{theorem}\label{th:linrel_regret}
  Assume that, for any $0<\delta<1$:
  \begin{align}\label{eq:betat}
  \beta_t=\max\left\{128nd\ln t\ln\frac{t^2}{\delta},\left(\frac{8}{3}\ln\frac{t^2}{\delta}\right)^2\right\},
  \end{align}
  then, for $\mathcal{C}_t=\mathcal{C}_t^2$:
  \begin{align}
  \mathrm{Pr}\left(\forall T,R(T)\leqslant n\sqrt{8nd\beta_TT\ln  \left(1+\frac{n}{d}T\right)}\right)\geqslant 1-\delta,
  \end{align}
   and, for $\mathcal{C}_t=\mathcal{C}_t^1$:
    \begin{align}
  \mathrm{Pr}\left(\forall T,R(T)\leqslant n^2d\sqrt{8\beta_TT\ln  \left(1+\frac{n}{d}T\right)}\right)\geqslant 1-\delta.
  \end{align}
  
\end{theorem}
The theorem is proved in Appendix~\ref{sec:linrel_proof}.

Thm.~\ref{th:linrel_regret} implies that \linrel, in both variants, applied to the social bandits case yields the same bound as in~\cite{dani2008stochastic}, i.e. a polylog bound of $\tilde{\mathcal O}(\sqrt{T})$, and which is known to be  tight~\cite{dani2008stochastic,abbasi2011improved}. Note that, compared to \cite{dani2008stochastic}, the bound is worse only by a factor of $n$, corresponding to the number of users in the network.
Though we provide bounds for both $\mathcal{C}^2$, $\mathcal{C}^1$, and the bounds for $\mathcal{C}^2$ are better than the ones for $\mathcal{C}^1$ by a factor of $nd$, there is a clear advantage for the $\mathcal{C}^1$ version: it makes \eqref{eq:opt} tractable. We discuss this next.

\subsection{An efficient solver of  \eqref{eq:opt} under $\mathcal{C}^1$ constraints}\label{sec:c1} Consider the case where the constraint set is $\mathcal{C}_t^1$, given by \eqref{c1}. In this case, we can solve \eqref{eq:opt} efficiently for several sets $\mathcal{B}$ of interest. In doing so, we exploit the fact that, for any $\zbb\in\reals^{nd}$,  the optimization problem:
\begin{subequations}\label{eq:linobj}
\begin{align}
\text{Maximize:} &\quad \ubb^\top \zbb\\
\text{subject to:} &  
\quad \|\hat{\ubb}_0-\ubb\|_{1,Z(t)}\leq c_t
\end{align}
\end{subequations}
attains its maximum at one of the $2nd$ extreme points of the polytope $ \|\ubb-\hat{\ubb}_0\|_{1,Z(t)}\leq c
$. These extreme points can be generated in polynomial time from $Z(t)$, and $\hat{\ubb}_0$. Hence, given $\zbb$, solving a problem of the form \eqref{eq:linobj} amounts to finding which of these $2nd$ extreme points yields the maximal inner product with $\zbb$.
This observation leads to the following means of solving \eqref{eq:opt}:
\begin{theorem}
Let $\mathcal{E}$ be the set of $2nd$ extreme points of the polytope $ \|\ubb-\hat{\ubb}_0\|_{1,Z(t)}\leq \sqrt{nd\beta_t}.$ Then, if $\mathcal{C}_t=\mathcal{C}_t^1$,  
\eqref{eq:opt} reduces to solving $2n^2d$ problems of the form:
\begin{subequations}%
\begin{align}
\text{Maximize:} &\quad \mathbf{z}^\top \mathbf{v}_i  \label{eq:simpleobj} \\
\text{subject to:} & \quad\mathbf{v}_i\in \mathcal{B},
 \end{align}
\label{eq:l1red2}%
\end{subequations}%
\noindent for $\mathbf{z}=(\ubb^T L(t))_i\in\reals^d$ the part of $\ubb^T L(t)$ corresponding to user $i\in [n]$, and $\ubb\in \mathcal{E}$.
\end{theorem}
\fullversion{}{
\begin{proof}
For a given $\vbb$, the optimal solution to $\max_{\ubb\in C_t^1}\ubb^\top L(t) \vbb $ is an element of $\mathcal{E}$, i.e.:
\begin{align*}\argmax_{\vbb\in \mathcal{B}^{(n)}} \max_{\ubb \in C_t^1} \ubb^\top L(t) \vbb =  \argmax_{\vbb\in \mathcal{B}^{(n)}} \max_{\ubb\in \mathcal{E}}\ubb^\top L(t) \vbb  
\end{align*}
Hence, the optimal solution to \eqref{eq:opt} can be found  by solving  $|\mathcal{E}|=2nd$ optimization problems of the form:%
\begin{subequations}%
\begin{align}%
\text{Maximize:} &\quad \ubb^\top  L(t)\vbb \\
\text{subject to:} & \quad\vbb\in \mathcal{B}^{(n)},
\label{uppl}\end{align}
\label{eq:l1red}%
\end{subequations}%
\noindent for all $2nd$ vectors $\ubb\in \mathcal{E}$. 
Note that in all cases the optimization \eqref{eq:l1red} is separable over $\mathbf{v}_i$, where $\vbb=[\mathbf{v}_1,\ldots,\mathbf{v}_n]$, and thus itself reduces to solving $n$ problems of the form \eqref{eq:l1red2}.
\end{proof} 
}
This reduction implies that we can solve \eqref{eq:opt} for \emph{a broad array of sets $\mathcal{B}$}. In particular: 
\begin{itemize}
    \item (a) When $\mathcal{B}$ is an arbitrary convex set (e.g., the Euclidian ball, a convex polytope, etc.), \eqref{eq:l1red2} is a convex optimization problem, so \eqref{eq:opt} amounts to solving $2n^2d$ convex problems. 
    \item{(b)} When $\mathcal{B}$ is a finite subset of $\reals^d$, the optimal solution to \eqref{eq:l1red2} can be found in polynomial time by finding the maximum among all $|\mathcal{B}|$ values \eqref{eq:l1red2}. Hence, a solution to \eqref{eq:opt} can be computed after a total of $2n^2d|\mathcal{B}|$ evaluations of a linear function--namely, \eqref{eq:simpleobj}.
\end{itemize} We stress here that the existence of efficient/polytime algorithms in the above cases is  remarkable precisely because of the exponential size of the possible combinations in \eqref{eq:opt}: this is most evident in the finite $\mathcal{B}$ case, where there are $|\mathcal{B}|^n$  candidate recommendation combinations across users.

\section{Thompson Sampling}
 
We turn now to a Bayesian interpretation and its associated algorithm, \ts. Instead of optimizing over a confidence ellipsoid or interval, \ts assumes  a prior on the parameter vector $\ubb_0$ and, at each step,  samples this  vector from the posterior obtained after the feedback has been observed. This way, the exploration is embedded in the uncertainty of the distribution being sampled.

\begin{algorithm}[!t]
  \caption{ -- \ts  }\label{alg:ts}
  \begin{algorithmic}[1]
    \REQUIRE{matrix $P$, parameter $\alpha$, item set $\mathcal{B}$, users $[n]$}
    \STATE{\textbf{Initialization:} play $d$ pulls for each $i\in[n]$ and observe rewards $\textbf{r}_0$}
    \STATE{$A(0)\gets\alpha I$}
    \STATE{Sample $\ubb_0(0)$ from $\mathcal{N}(\hat{\ubb}_0(0),\Sigma(0))$}
    \FOR{$t = 1, \ldots, T$}
    \STATE{estimate $$\hat{\ubb}_0(t) = \argmin_{\ubb\in \reals^{nd}}\sum_{\tau=1}^{t-1} \|X(V(\tau),A(\tau))\ubb - \mathbf{r}(\tau)\|_2^2$$}
    \STATE{ Sample $\ubb$ from $\mathcal{N}(\hat{\ubb}_0(t),\Sigma(t))$}
    \STATE{ Recommend  $\vbb_t=\argmax_{\vbb\in\mathcal{B}^{(n)}}\ubb^\top L(t)\vbb$}
    \STATE{Observe reward $\mathbf{r}(t)$, update $A(t)$ and $\Sigma(t+1)$}
    \ENDFOR   
  \end{algorithmic}
\end{algorithm}
\subsection{Algorithm overview}Algorithm~\ref{alg:ts} outlines \ts.
 We assume that our parameter vector, $\ubb_0$, is distributed as a multi-variate normal having prior $\mathcal{N}\left(\ubb(0),\Sigma(0)\right)$ -- obtained, for instance, after playing the initialization arms. At each $t$, the covariance  is updated via:
\[
  \Sigma(t+1)=\left(\Sigma(t)^{-1}\!+\!\frac{X\left(V(t),A(t)\right)^\top X\left(V(t),A(t)\right)}{\sigma^2}\right)^{-1}\!\!\!,
\]
where $\sigma^2$ is the standard deviation of the reward noise. Then, $\ubb(t)$ is obtained by sampling from the distribution $\mathcal{N}(\hat{\ubb}_0(t),\Sigma(t))$. Finally, $\vbb(t+1)$ is chosen by solving the (simpler than \eqref{eq:opt}) optimization:
\begin{align}
    \vbb(t+1)&=\argmin_{\vbb\in\mathcal{B}^{(n)}}\hat{\ubb}(t+1)^\top L(t+1)\vbb.
\end{align}

\subsection{Regret}
The quantity of interest for \ts is the \emph{Bayesian regret}, i.e., the aggregate regret in expectation:
\begin{align}
\text{BR}(T)&= \mathbf{E}\left[\sum_{t=1}^T\left(\ubb_0^\top L(t)\vbb_*(t)-\ubb_0^\top L(t)\vbb(t)\right)\right].
\end{align}
An $\tilde{\mathcal O}(\sqrt{T})$ bound applies to \ts in our setting:
\begin{theorem}\label{th:ts_regret}
  Assume that, for any $0<\delta<1$, $\beta_T$ is set as:
  \[
    \beta_T= 1+\sqrt{2\log\frac{1}{\delta}+nd\log\left(1+\frac{n}{d}T\right)}.
  \]
  Then:
  \[
  \mathrm{Pr}\left(\forall T,\mathrm{BR}(T)\leqslant 2\!+\!2n\beta_TT\sqrt{2ndT\ln\left(1\!+\!\frac{n}{d}T\right)}\right)\!\geqslant\! 1\!-\!\delta.
  \]
\end{theorem}
The proof can be found in Appendix~\ref{sec:ts_proof}.
\fullversion{}{
\subsection{Efficient solver}
In our case, \ts is implemented as described in~\cite{russo2018tutorial}, in that we do not re-compute $\hat{\ubb}$ at every step, but we maintain the sample incrementally as follows:
\begin{align}
\ubb(t\!+\!1)&=\Sigma(t\!+\!1)\left(\Sigma(t)^{-1}\ubb(t)\!+\!\frac{L(t)\vbb(t)\left(\mathbf{r}(t)\!+\!\mathbf{\tilde{w}}(t)\right)}{\sigma^2}\right),
\end{align}
where $\mathbf{\tilde{w}}(t)$ is sampled from $\mathcal{N}(0,\sigma^2I)$.
This ensures that \ts is more efficient than \linrel: only the inversion of $\Sigma(t)$ is $\mathcal{O}(nd^{2.3})$;  the rest of the operations are either matrix multiplications or sampling from multi-variate normal distributions in which dimensions are independent (the noise). In terms of the optimization over $\mathcal{B}^{(n)}$, the same observations as in Section~\ref{sec:c1} apply here: we can efficiently optimize over finite or convex sets.
}

\section{\linucb Algorithm}\label{sec:linucb}

\linucb \cite{li2017provably} is an alternative to \linrel for linear bandits. Unfortunately, in our case, it leads to an intractable problem.
%
%
Applied to our setting, \linucb is identical to Alg.~\ref{alg:linrel}, with only a change in how recommendations are selected. That is,  $\hat{\ubb}_0(t)$ is again estimated via least squares estimation \eqref{eq:lse} as in \linrel. For  $\Sigma(t) = Z^{-1}(t)$  the covariance of the estimate $\hat{\ubb}^0$, \linucb  selects
\begin{align}
\vbb_t&\!=\!\argmax_{\vbb\in (\mathcal{B})^n} \left(\hat{\ubb}_0^\top(t) L(t) \vbb  \!+\! c  \vbb^\top L(t)^\top \Sigma(t) L(t) \vbb \right).\!\!\! \label{eq:linucb_select}\end{align}
\linucb can thus be implemented by replacing line~\ref{line:optimization} of Algorithm~\ref{alg:linrel} with equation~\eqref{eq:linucb_select}.

Unfortunately, \eqref{eq:linucb_select} is a non-convex optimization, with no obvious solution for different cases of $B$. In the case where
$\mathcal{B}_t=\mathbb{B}=\{v\in\reals^d: \|v\|\leq 1\}$,
an approximate solution can be constructed via an SDP relaxation, as in~\cite{lu2014optimal}; the steps are detailed in  Appendix \ref{app:linucb}.
Nevertheless, this solution is only within a constant approximation of the optimal; as such, it cannot be used to bound the regret.

%% file: experiments.tex
\section{Experiments}\label{sec:exp}

To validate our analysis, we evaluated the regret of the \linrel, \ts, and \linucb algorithms for both synthetic and real-life data. We describe experiments on both types of datasets below.

\begin{figure*}
    \centering
    
    \subfloat[Regret, finite set $n=10$, $d=5$, $m=100$]{\includegraphics[width=0.3\textwidth]{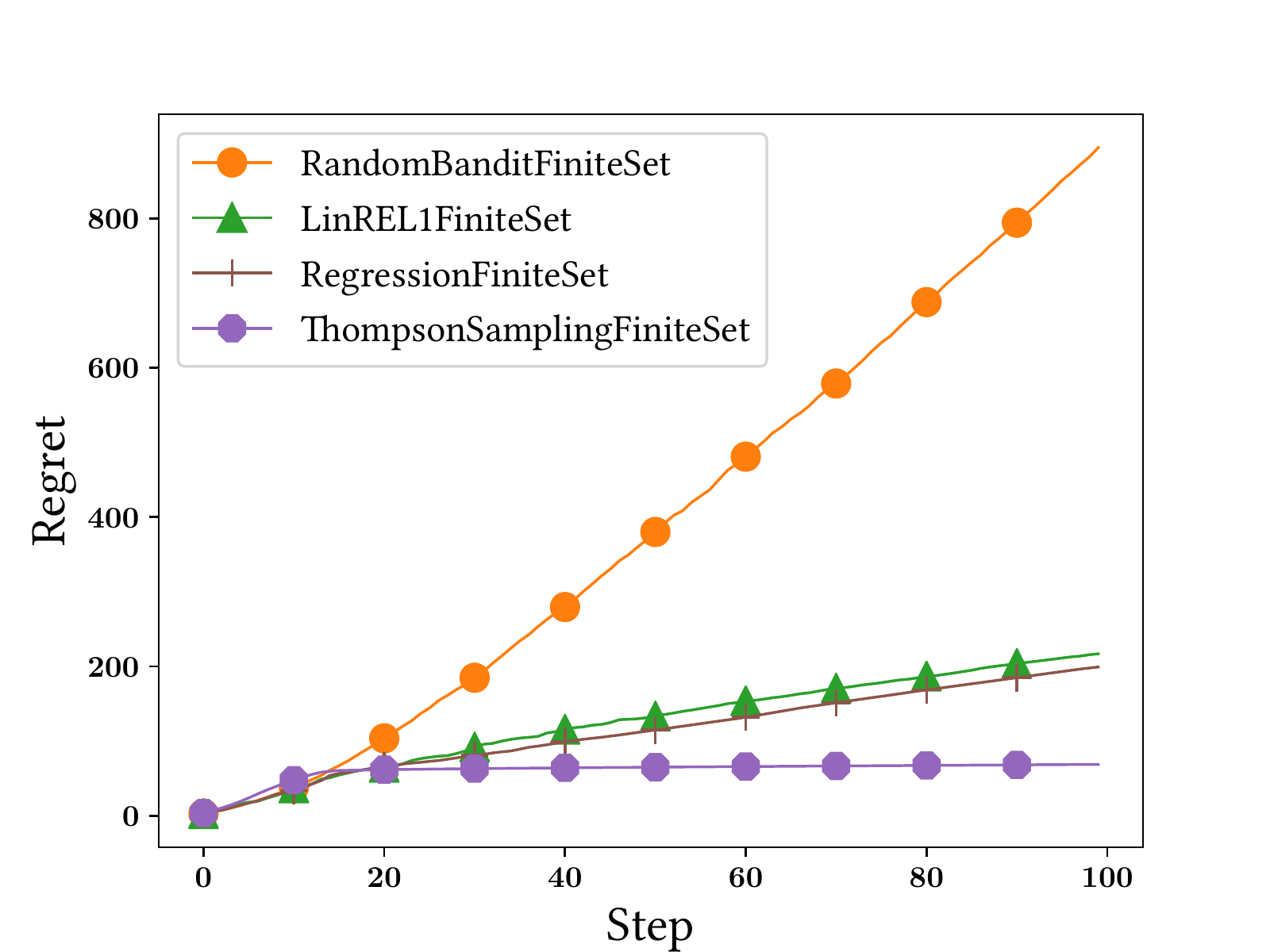}}
    ~
    \subfloat[Regret, finite set $n=100$, $d=20$, $m=1000$]{\includegraphics[width=0.3\textwidth]{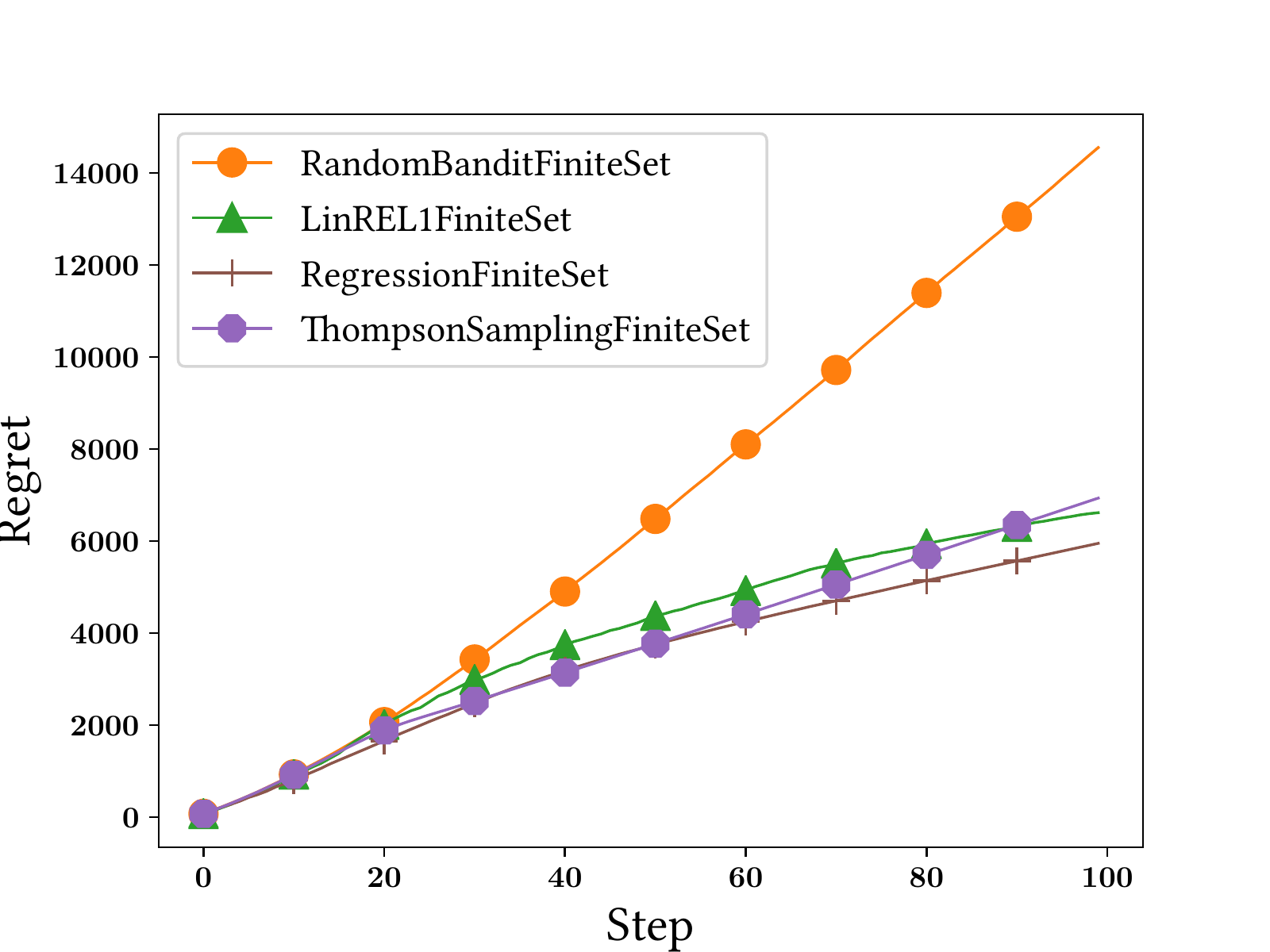}}
    ~
    \subfloat[Time, finite set $n=100$, $d=20$, $M=100$ (log $y$-axis)]{\includegraphics[width=0.3\textwidth]{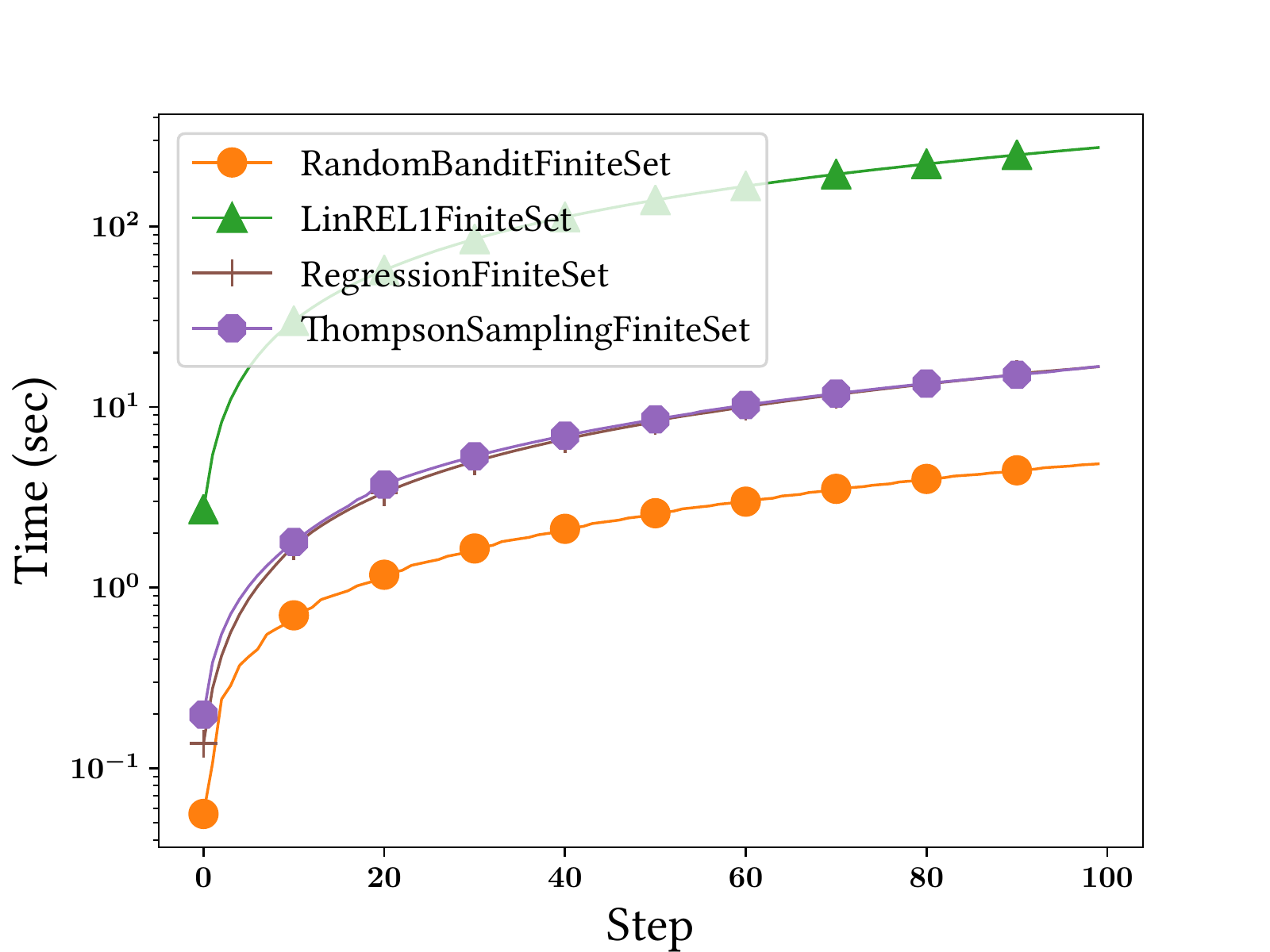}}
    
    \subfloat[Regret, $L_2$ ball $n=10$, $d=5$]{\includegraphics[width=0.3\textwidth]{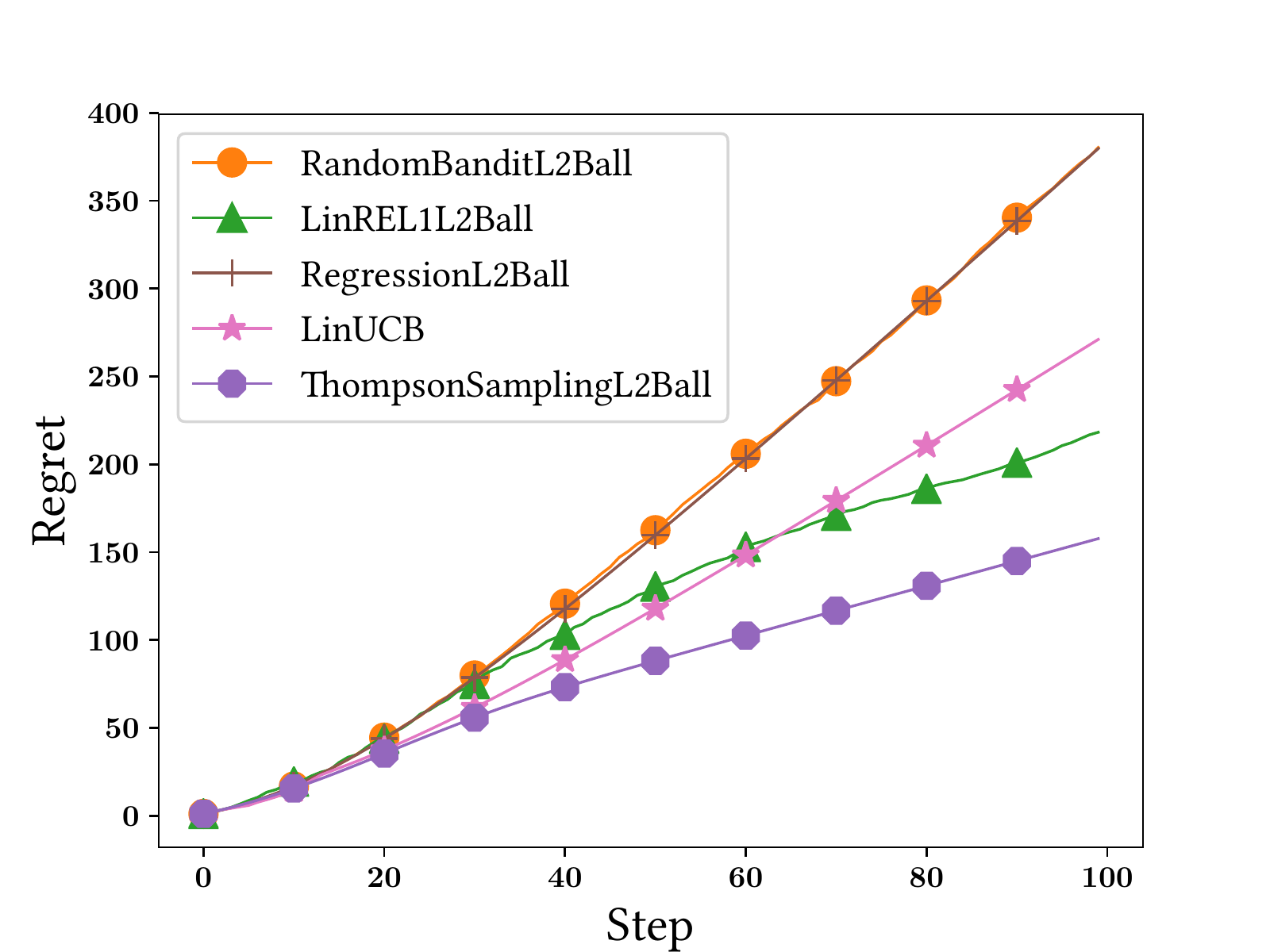}}
    ~
    \subfloat[Regret, $L_2$ ball $n=100$, $d=20$]{\includegraphics[width=0.3\textwidth]{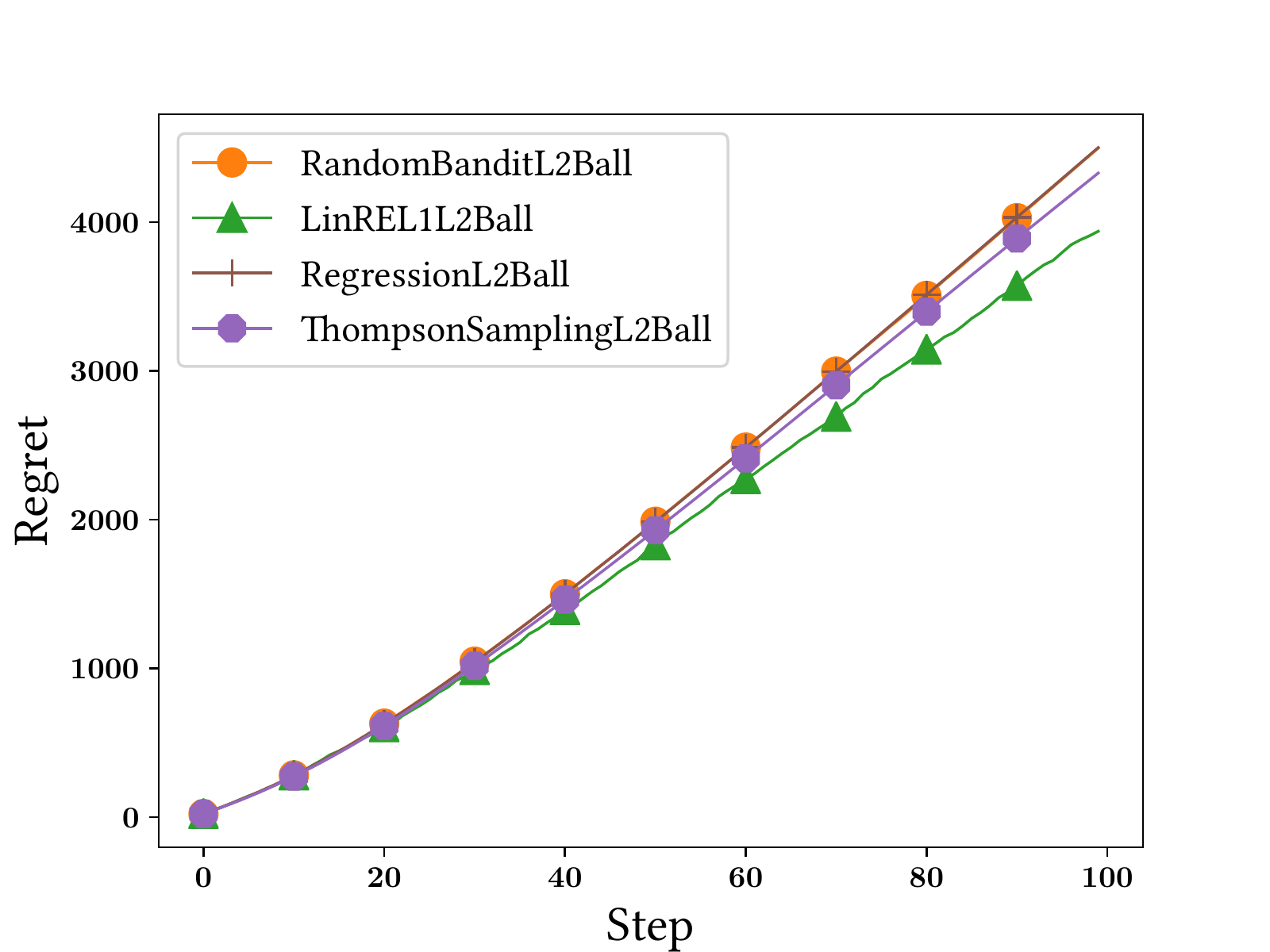}}
    ~
    \subfloat[Time, $L_2$ ball $n=10$, $d=5$ (log $y$-axis)]{\includegraphics[width=0.3\textwidth]{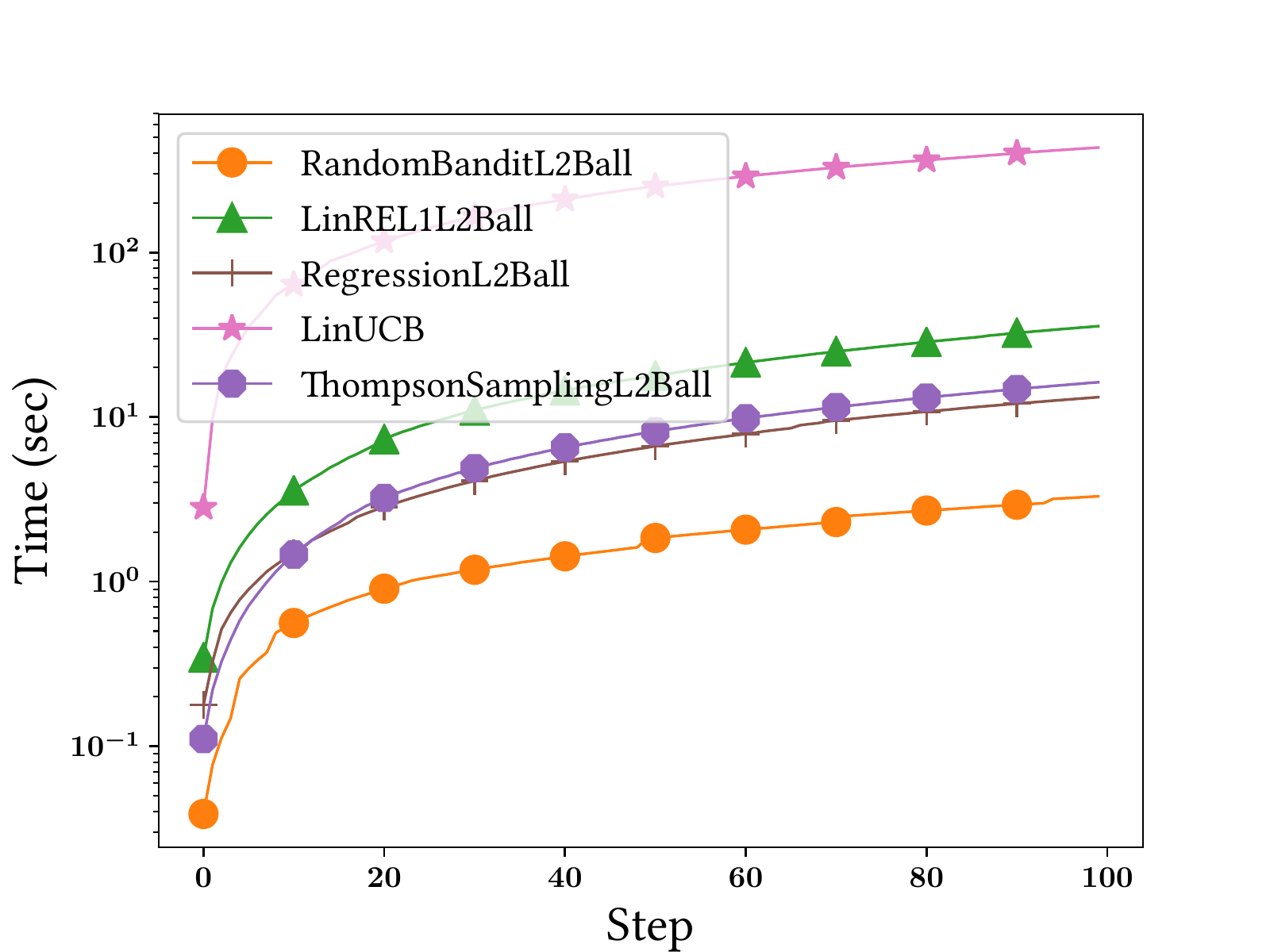}}
    
    \caption{Regret and execution time analysis. Setup: horizon $100$, $\sigma=1$, $10^{-5}$ 
        $\beta$ scale}\label{fig:regret_time_syn}
\end{figure*}

\subsection{Synthetic Data.} 
\noindent\textbf{Experiment Setup.} 
\fullversion{We generate a complete network (CMP) on which influence probabilities -- representing the matrix $P$ -- were assigned between the $n^2$ pairs each having a value of $1/n$.}{
We  generate three different types of synthetic networks from which $P$ was constructed:
\begin{enumerate}
\item a complete network (CMP) on which influence probabilities were assigned between the $n^2$ pairs each having a value of $1/n$
\item an Erd\H{o}s-R\'enyi graph (ER) having $p=\frac{\log n}{n}$, corresponding to the setting in which the graph contains a single connected component w.h.p.; the influence probability on each edge $(v_i,v_j)$ was set as $1/\text{deg}(v_i)$
\item a Barab\'asi-Albert (BA) graph, with parameter $m=\log n$, i.e., the same average degree as in ER, and $P$ was generated as in the case of ER. 
\end{enumerate}
We generated graphs of different sizes  $n$ given by Table~\ref{tab:params}.}
In our experiments on synthetic data, we set $\alpha$ to be $0.05$.  \fullversion{}{We constructed $P$ as detailed above.  We consider recommendation profiles of  $d$ again spanning multiple values shown in Table~\ref{tab:params}.} We consider two cases for the set of recommendations $\mathcal{B}$: (a) a finite catalog of size $M=|\mathcal{B}|$, for $M$ \fullversion{in $\{10,100\}$}{given in Table~\ref{tab:params}}, and (b) the unit ball $\mathbb{B}$. In the case of a finite catalog, we generate $M$ random vectors in $[0{,}1]^d$ by uniform sampling. Finally we generate inherent profiles by sampling $n$ vectors uniformly from $[0{,}1]^d$.
Given a recommendation we assume that ratings by users are generated via \eqref{factor}, where i.i.d.~noise $\varepsilon$ has standard deviation %
\fullversion{$1$}{$\sigma$. We also explore different values of $\sigma$, as indicated in Table~\ref{tab:params}.}

\noindent\textbf{Algorithms.}
We implement\footnote{Our code is available at \texttt{\url{https://git.io/JUuSE}}.}
\linrel with $\mathcal{C}^1_t$ constraints (\linrel1),   \ts, and \linucb. 
In the case of \linrel, we multiply the theoretically designed radius $\beta_t$, given by \eqref{eq:betat}, with a scaling factor $\beta$; \fullversion{we set it to $\beta=10^{-5}$ after experimental validation}{we experiment with several values of $\beta$, as shown in Table~\ref{tab:params}}.  

We compare to two baselines: (a) the usual random bandit (\textsc{Rand}), which explores without exploiting by choosing a vector randomly, and (b) a regression  baseline (\textsc{Regres\-sion}) which exploits without exploring: at each step, it only performs the least squares estimation of profiles and selects optimal recommendations in $\mathcal{B}$, but does not do any exploration. 
In all cases, we set the number of rounds (the horizon $T$) to $100$.

\noindent\textbf{Regret and Execution Time.}
We compare the regret attained by different algorithms over topology CMP, in Figure~\ref{fig:regret_time_syn}. Subfigures (a),(b) show the regret as a function of iterations under a finite set catalog $\mathcal{B}$, and subfigures (d),(e) show the regret as a function of iterations when the catalog is the $\ell_2$ ball $\mathbb{B}$.   Compared to the baselines, we find that our bandit algorithms significantly outperform \textsc{Rand} and are competitive with ``reasonable'' baseline \textsc{Regression}. In the finite case, their difference from the latter is less pronounced. However, in the $\ell_2$ case, \textsc{Regression} becomes almost as ineffective as the random baseline. \ts is always at least competitive with \linrel, and in some cases, especially for low values of $n$ and $d$, seems to be the best variant. In the only test we could run for \linucb we found that it is worse than \linrel and \ts.

When looking at the execution time (Figure~\ref{fig:regret_time_syn}c,f), it becomes clear that first, \textsc{LinUCB} is too costly for its regret performance. This is due to the SDP relaxation involved in arm selection. \linrel has a more reasonable running time, but the best trade-off is attained by \textsc{ThompsonSampling}.

\fullversion{}{
\begin{table}[!t]\centering
    \caption{Experimental parameter settings. Values in bold are the ones used by default unless otherwise specified.}\label{tab:params}
    \small
    \begin{tabular}{cr}
        \toprule
        Parameter&Values\\
        \midrule
        $n$&$\{2,\mathbf{10},100\}$\\
        $d$&$\{2,\mathbf{5},10,20\}$\\
        $\alpha$&$\{0.01,0.05,\mathbf{0.1},0.2,0.3,0.4,0.5\}$\\
        $M=|\mathcal{B}|$&$\{10,\mathbf{100},1000\}$\\
        $\sigma$&$\{0.1,\mathbf{1},2\}$\\
        $\beta$ scaling factor&$\{1, 10^{-1},\dots, \mathbf{10^{-5}},10^{-6}\}$\\
        \bottomrule
    \end{tabular}
\end{table}
}

%

\fullversion{}{
\noindent\textbf{Impact of Catalog Size.} Table~\ref{tab:comp_items} shows the effect of the catalog size $M$ on the regret for different algorithms. we see that the bandit-based algorithms are generally better than the \textsc{Regression} baseline, however which one is the best can vary between the different settings of $M$. Not surprisingly, reducing $M$ gives more advantage to \textsc{Regression}, as there is less benefit to exploration. When options become abundant however, \textsc{Regression} becomes worse than the bandit methods.

\noindent\textbf{Impact of Noise.} We observe a similar phenomenon when studying the impact of noise variance. Table~\ref{tab:comp_sigma}  shows the impact that noise in responses, as captured by $\sigma$,  has on the regret of different algorithms.  We observe that, in the case of finite sets and for low noise ($\sigma=0.1$), there is no need for exploration and \textsc{Regression} performs best. On the other hand, the usefulness of the bandit-based algorithms is clear for higher values.

\begin{table}[!t]\centering
  \caption{Regret at horizon $100$ obtained for various values of $M$, finite case; lower values are better.}\label{tab:comp_items}
  \small
  \begin{tabular}{|c|rrr|}
    \hline
    Method&$M=10$&$M=100$&$M=1000$\\
    \hline
    \textsc{LinREL1}& $138.00$ & $216.76$ & $\mathbf{278.87}$\\
    \textsc{ThompsonSampling}& $212.23$ & $\mathbf{68.64}$ & $450.66$\\
    \textsc{Regression}& $\mathbf{28.69}$ & $199.11$ & $400.24$\\
    \hline
  \end{tabular}
\end{table}

\begin{table*}[!t]\centering
  \caption{Regret at horizon $100$ obtained for various $\sigma$ values, $M=100$; lower values are better.}\label{tab:comp_sigma}
  \begin{tabular}{|c|rrr|rrr|}
    \hline
    & \multicolumn{3}{c}{Finite Set} \vline& \multicolumn{3}{c}{$L_2$ Ball}\vline\\
    Method&$\sigma=2$&$\sigma=1$&$\sigma=0.1$&$\sigma=2$&$\sigma=1$&$\sigma=0.1$\\
    \hline
    \textsc{LinREL1}& $428.33$ & $216.76$ & $139.37$& $314.07$ & $218.24$ & $107.82$\\
    \textsc{TSamp.}& $\mathbf{83.39}$ & $\mathbf{68.64}$ & $47.48$& $\mathbf{159.50}$ & $\mathbf{157.71}$ & $\mathbf{136.23}$\\
    \textsc{Regr.}& $402.83$ & $199.11$ & $\boldsymbol{27.93}$& $379.71$ & $379.71$ & $379.71$\\
    \hline
  \end{tabular}
\end{table*}

\begin{table*}\centering
  \caption{Regret at horizon $100$ obtained for various methods of computing $A$ (expectation default, stochastic, fix-point $A_\infty$); lower values are better.}\label{tab:comp_types}
  \begin{tabular}{|c|rrr|rrr|}
    \hline
    & \multicolumn{3}{c}{Finite Set} \vline& \multicolumn{3}{c}{$L_2$ Ball}\vline\\
    Method& &stoch.&$A_\infty$& &stoch.&$A_\infty$\\
    \hline
    \textsc{LinREL1}& $216.76$ & $261.07$ & $397.07$& $217.24$ & $\mathbf{238.15}$ & $\mathbf{242.80}$\\
    \textsc{ThompsonSampling}& $\mathbf{68.64}$ & $\mathbf{40.65}$ & $\mathbf{281.07}$& $\mathbf{157.71}$ & $285.57$ & $446.83$\\
    \textsc{Regression}& $199.11$ & $243.26$ & $294.89$& $379.71$ & $457.00$ & $446.83$\\
    \hline
  \end{tabular}
\end{table*}

\noindent\textbf{Impact of Social Network Effects.} In our setting, the neighbor effects are modeled via the $\alpha$ parameter, which encodes how much the original profile of each user matters in the recommendation process. The higher $\alpha$ is, the less the social network matters; in all other experiments this value is set to $0.05$ -- in line with other Markov chain-style models using teleportation vectors. To give the full picture, we present in Figure~\ref{fig:comp_alpha} the comparison of our various algorithms with the random and regression baselines, for a selection of $\alpha$ values. On thing that is immediately apparent is that \linrel is quite stable for different values, while \ts seems to degrade as $\alpha$ goes higher. This suggests that \linrel may be a better overall option when $\alpha$ is high, while $\ts$ is generally betyter for low values (under $0.15$).

\begin{figure}\centering
	\subfloat[$\alpha$ values, finite set]{\includegraphics[width=0.45\textwidth]{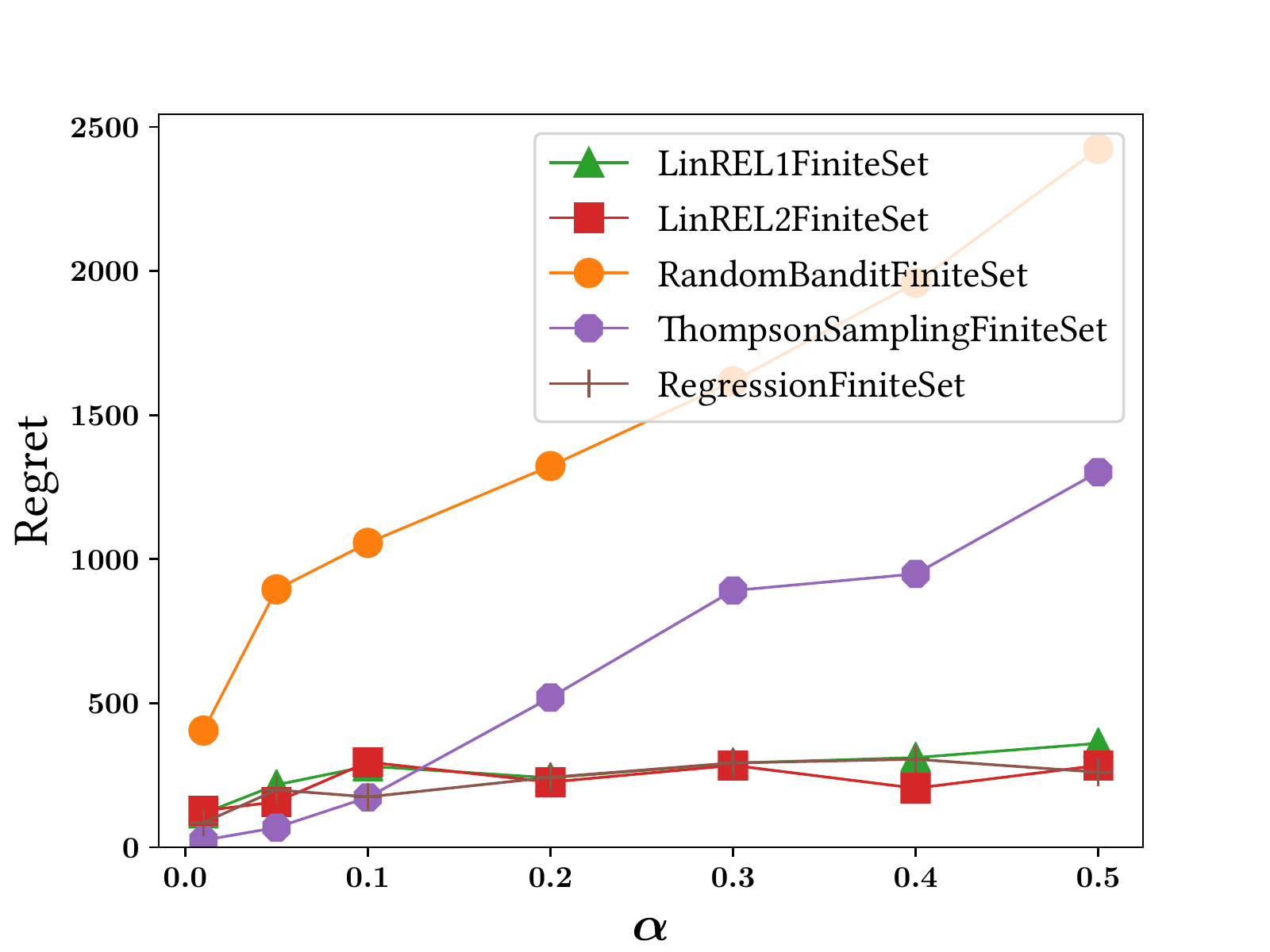}}
	~
	\subfloat[$\alpha$ values, $L_2$ ball]{\includegraphics[width=0.45\textwidth]{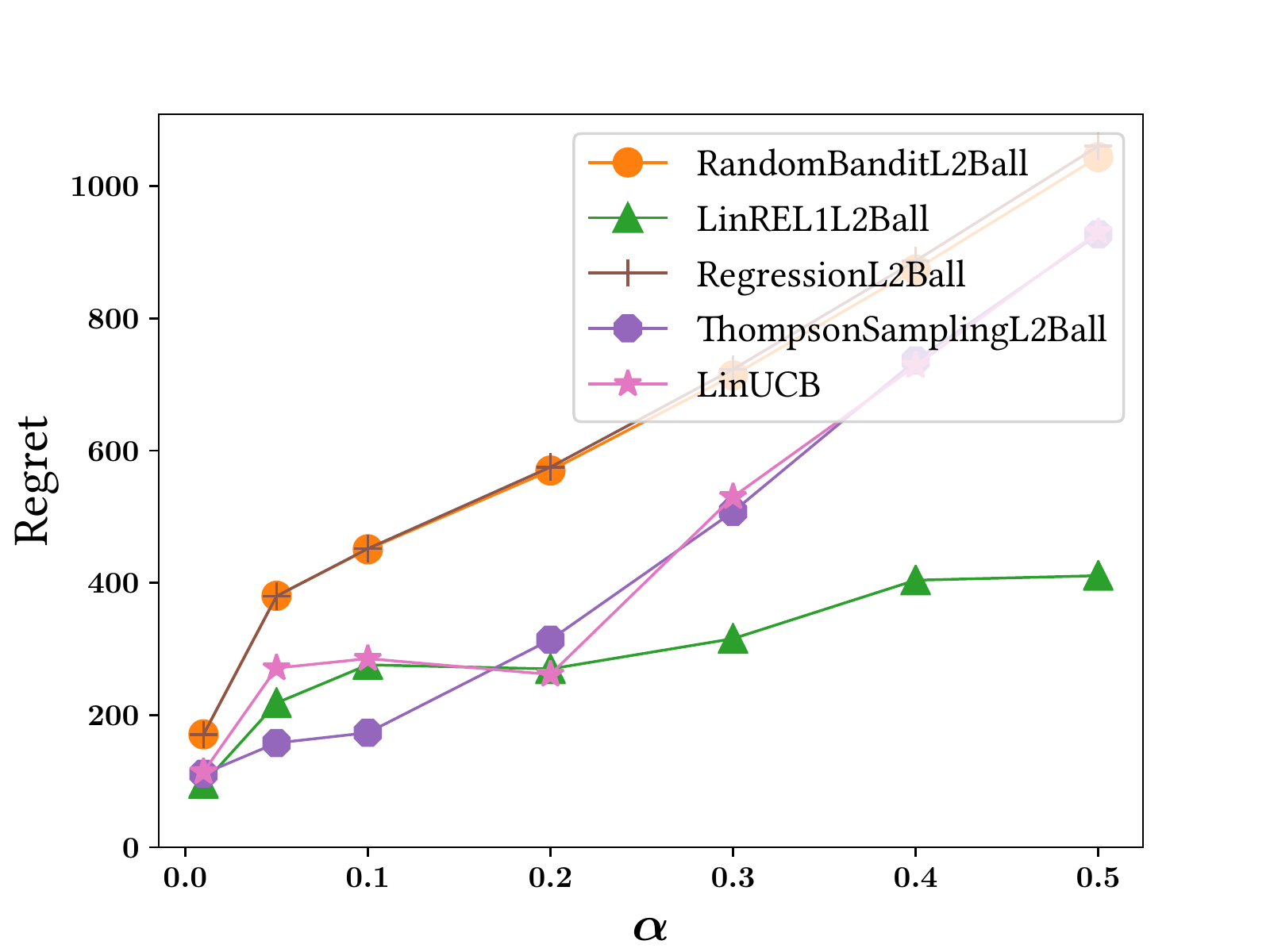}}
	
	\caption{Regret at horizon $100$ obtained for various values of $\alpha$; lower values are better.}\label{fig:comp_alpha}
\end{figure}

\noindent\textbf{Benefit of Tracking Dynamics}. Our model in Section~\ref{sec:problemformulation}, and resulting algorithms, operate under the assumption that user interests at each step evolve according to the dynamics prescribed by \eqref{evol}. 
We test here the behavior of our algorithms under the assumption that interests evolve \emph{stochastically}, i.e.,  with probability $\alpha$ one chooses $U_0$ and the social choice only occurs with $1-\alpha$ probability. In this setting, \eqref{evol} captures the expected (not the exact) evolution of interest dynamics.

We also study the case where the bandits do not take $A(t)$ at each step into account (i.e., the evolution of interests due to time, but assume that the system has converged to steady state. That is, the design matrix capturing interactions between users is $$A_\infty=\lim_{t\to\infty}A(t)\stackrel{\eqref{eq:social_update}}{=}\alpha(I -(1-\alpha)P)^{-1}$$ one takes the fixed point of $A$, denoted $A_\infty$. This reduces the bandit algorithms to classic linear bandits, as there is no longer a time-evolving interaction.

Table~\ref{tab:comp_types} shows that bandit algorithms are less sensitive to stochasticity than  \textsc{Regression} method; this is expected, as adding stochasticity in some sense increases variance which, in turn, makes the need for exploration more prominent.  Moreover, ignoring dynamics (i.e., using steady state influences $A_\infty$) is always worse than tracking them. 

\noindent\textbf{Comparison of Different Networks}. The same observation as above in general hold when comparing the three methods of generating the synthetic social networks, as shown in Table~\ref{tab:comp_nets}: bandits outperform the \textsc{Regression} baseline. 

\begin{table*}\centering
  \caption{Regret at horizon $100$ obtained for various synthetic graph models,  $n=100$, $M=1000$.}\label{tab:comp_nets}
  \begin{tabular}{|c|rrr|rrr|}
    \hline
    & \multicolumn{3}{c}{Finite Set} \vline& \multicolumn{3}{c}{$L_2$ Ball}\vline\\
    Method& cmp.& ER & BA& cmp.& ER & BA\\
    \hline
    \textsc{LinREL1}& $2587.88$ & $\mathbf{2667.86}$ & $2654.46$ & $1538.59$ & $\mathbf{1616.38}$ & $1584.60$\\
    \textsc{ThompsonSampling}& $\mathbf{885.68}$ & $3239.83$ & $\mathbf{2352.06}$& $\mathbf{1147.94}$ & $2086.91$ & $\mathbf{1215.80}$\\
    \textsc{Regression}& $2960.91$ & $2772.71$ & $2633.93$ & $2007.14$ & $2086.91$ & $1937.72$\\
    \hline
  \end{tabular}
\end{table*}

}

\fullversion{}{
\noindent\textbf{Impact of Scaling Factor $\beta$.} We observed that in practice $\beta_t$ overestimates the diameter of the confidence ball. When applying \emph{scaling factor} $\beta$, the value range of the regret is shown in Table~\ref{tab:params}. We compare in Figure~\ref{fig:comp_scale} how this scaling factor influences the \linrel regret. The regret decreases with the scaling factor, it reaches a low value at around $10^{-5}$ and then starts slowly increasing, suggesting that having a ellipsoid to optimize around is still beneficial for \linrel.

\begin{figure}\centering
  \subfloat[$\beta$ scale, finite set]{\includegraphics[width=0.45\textwidth]{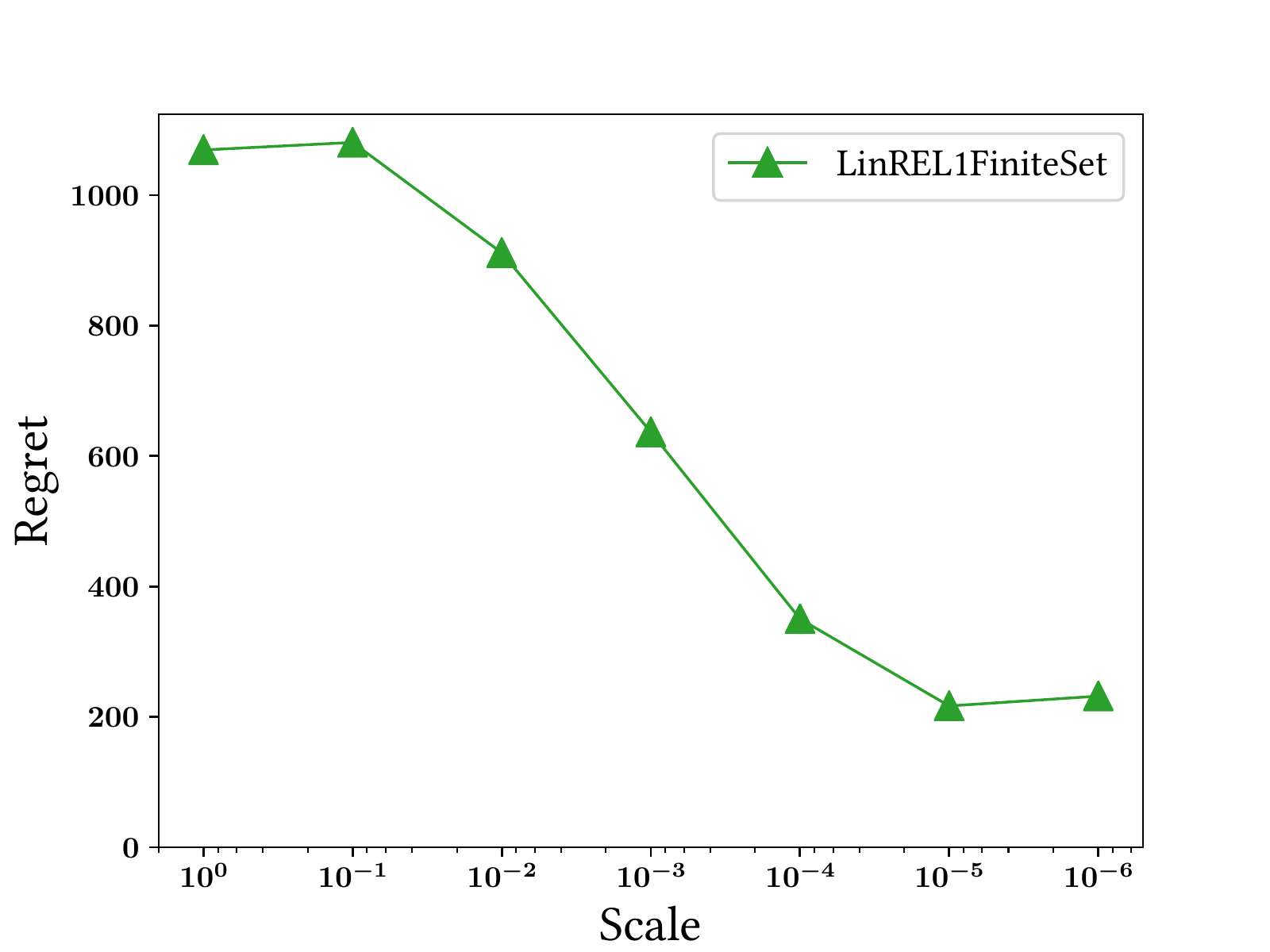}}
  ~
  \subfloat[$\beta$ scale, $L_2$ ball]{\includegraphics[width=0.45\textwidth]{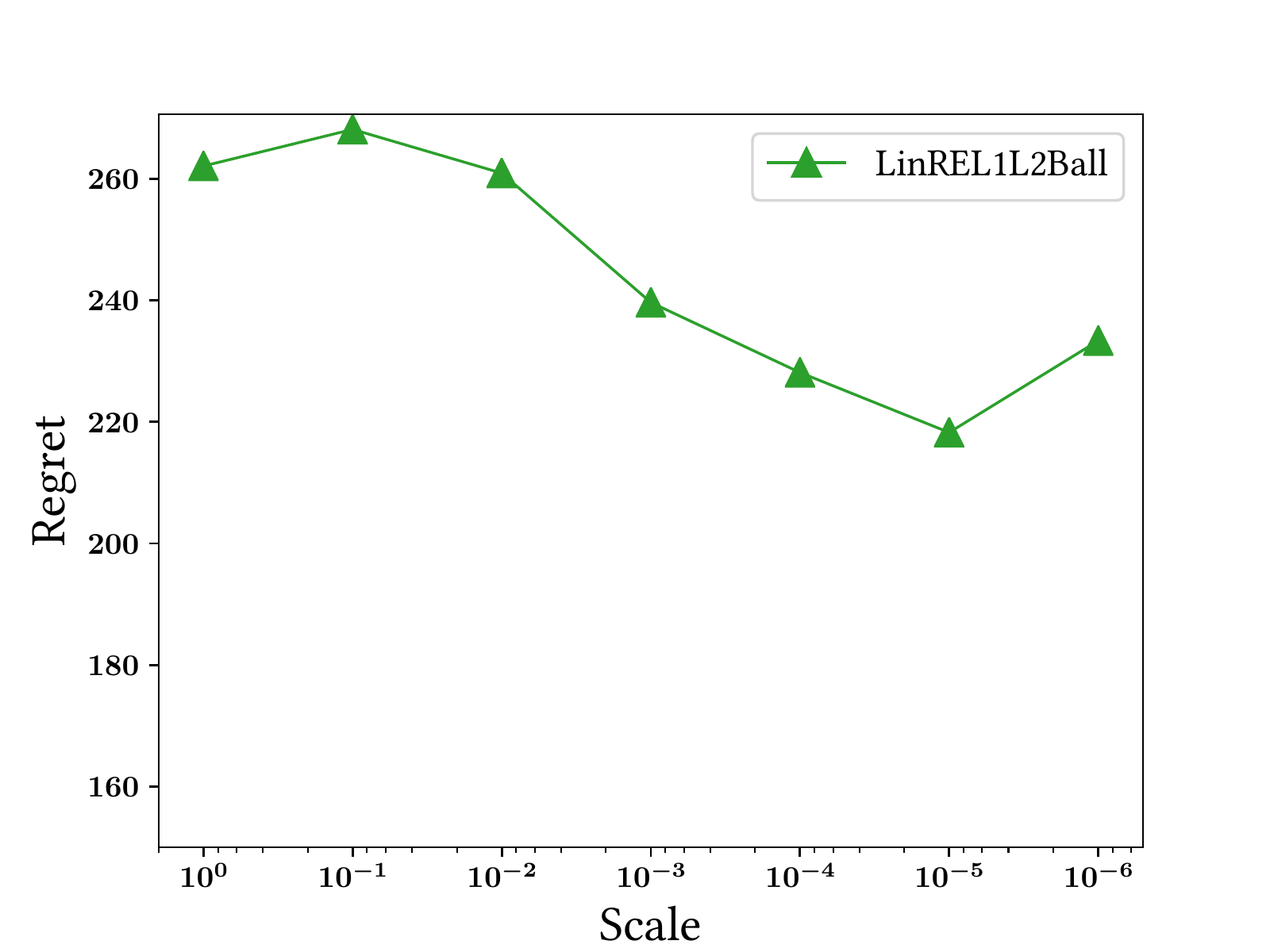}}
  
  \caption{Regret at horizon $100$ obtained for various values of the $\beta$ scale; lower values are better.}\label{fig:comp_scale}
\end{figure}
}

\subsection{Real Data}
 Finally, to validate the recommendation algorithms on real data, we used data from the now defunct Flixster site, a social network for movie discovery and reviews. The dataset contains $1{,}049{,}492$ users in a social network of $7{,}058{,}819$ contact links, $74{,}240$ movies that have generated $8{,}196{,}077$ reviews, each giving $1$ to $5$ stars. Each movie description includes one or more categories out of $28$ possible; each category is thus a dimension in the profile and item vectors.
 
To adapt this dataset to our setting, and to allow reasonable execution time for all methods, we have removed all users not having at least $1{,}500$ reviews. For each resulting user, we have evaluated the $\mathbf{u}_0$ vectors of $d=28$ by linear regression on their respective reviews, where the feature vector contains the $28$ categories and the $1$ to $5$ star ratings are mapped into the $[-1,1]$ interval. This resulted in a dataset of $n=206$ users and $d=28$. $P$ was generated by extracting the subgraph in the original network corresponding to the $206$ remaining users, and setting the influence probability \fullversion{}{as for ER and BA}, $1/\text{deg}(v_i)$. The finite set $M=100$ was generated, as in the synthetic case, by uniformly sampling each from $[0{,}1]^d$; we have not used real vectors, so as not to bias the recommendation (the $u_0$ vectors were generated via regression on real items, as described above).

\begin{figure}
  \centering
  
  \includegraphics[width=0.6\textwidth]{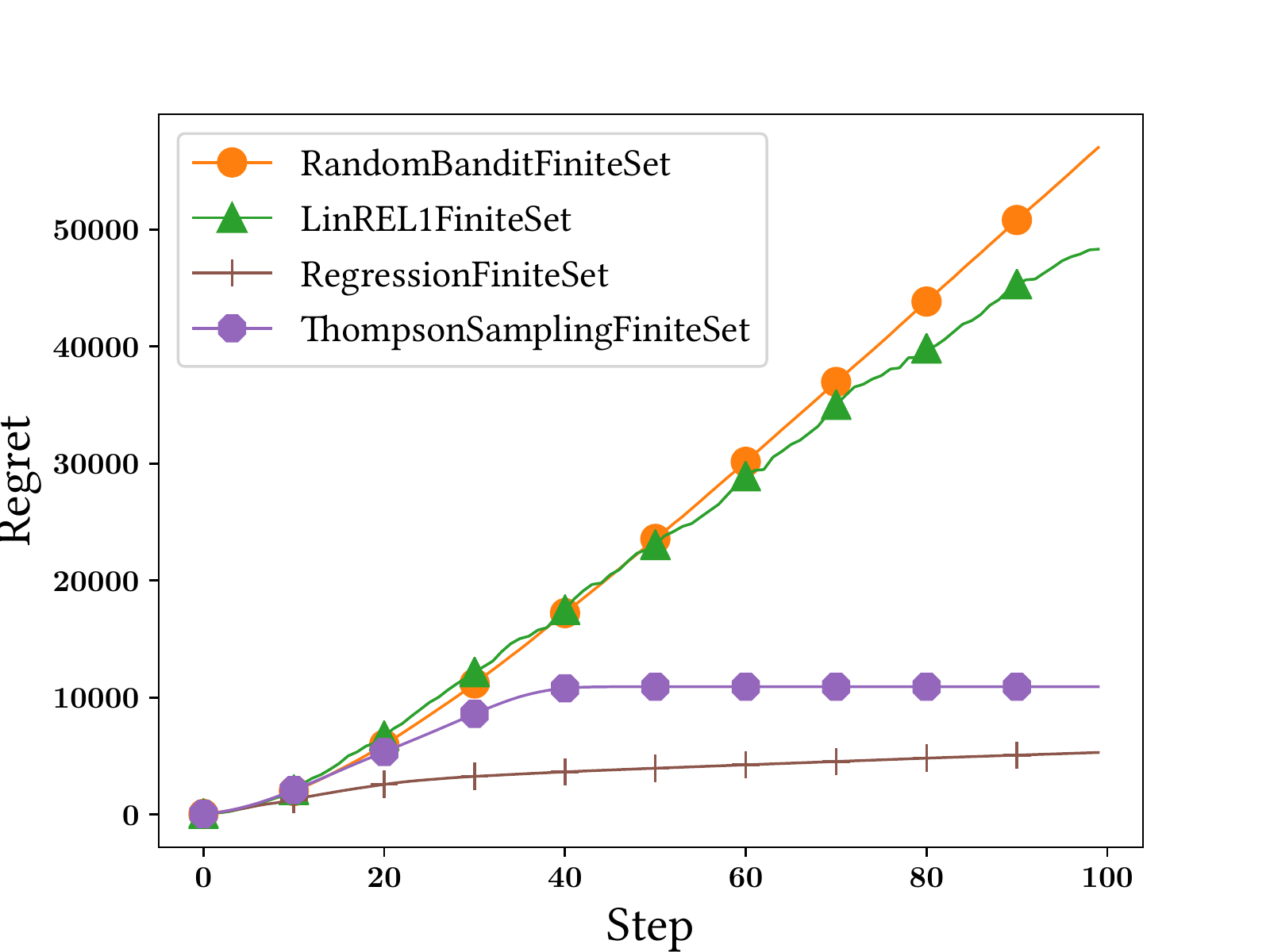}
  
  \caption{Flixstr regret $n=206$, $d=28$, $M=100$, $\sigma=1$}\label{fig:regret_flixstr}
\end{figure}

Figure~\ref{fig:regret_flixstr} presents the regret of the bandit algorithms. It can be seen that, in line with results on the synthetic data, \ts attains the best trade-off. The worse performance of \linrel is most likely due to a confidence ellipsoid that is too large. The good performance of \textsc{Regression} is an artifact of the experiment -- as said above, the $u_0$ vectors are generated via linear regression; in that sense, it can be considered a lower bound on the possible regret.

%% file: acknowledgements.tex
\section*{Acknowledgments}

Silviu Maniu and Bogdan Cautis gratefully acknowledge support from project PSPC AIDA (2019-PSPC-09), and Stratis Ioannidis gratefully acknowledges support from the NSF (grants  CCF-1750539 and CCF-1937500).

%% file: proof.tex



\section{Proof of Theorem~\ref{th:linrel_regret}}\label{sec:linrel_proof}

In the following, we use the general flow of the proof in \cite{dani2008stochastic} for the \textsc{ConfidenceBall}$_2$ algorithm, with the improvements used in~\cite{abbasi2011improved}.

The proof has two main ingredients:
\begin{enumerate}
  \item Assuming that $\ubb$ is in the $\mathcal{C}^2_{t}$ ball, then the regret is bounded.
  \item With high probability, $\ubb$ is in the $\mathcal{C}^2_{t}$ ball.
\end{enumerate}

\begin{lemma}
  If $\ubb\in\mathcal{C}_t^2$, 
  \[
  R(T)\leqslant\sqrt{8n^3d\beta_TT\ln T}.
  \]
  If $\ubb\in\mathcal{C}_t^1$,
  \[
  R(T)\leqslant nd\sqrt{8n^2\beta_TT\ln \log \left(1+\frac{n}{d}T\right)}.
  \]
\end{lemma}

\begin{proof}
  
Recall that:
\begin{align*}
  Z(t) &=\sum_{\tau=1}^{t-1}X(V(\tau),A(\tau))^\top X(V(\tau),A(\tau))
  = \sum_{\tau=1}^{t-1} \sum_{i=1}^n \x_i(\tau)  \x_i(\tau)^\top.
\end{align*}
Further denote as $\ubb(t)$ the choice of $\ubb$ at time $t$. We need to bound the following difference(s) (the $t$ is ignored in the following):
\begin{align*}
|(\ubb_0-\ubb)^\top\x_i|&=|(\ubb_0-\ubb)^\top Z(t)^{1/2}Z(t)^{-1/2}\x_i|\\
&=\|(Z(t)^{1/2}(\ubb_0-\ubb)^\top Z(t)^{-1/2}\x_i\|\\
&\leqslant \|Z(t)^{1/2}(\ubb_0-\ubb)\|_2\cdot\|Z(t)^{-1/2}\x_i\|_2\\
&=\|Z(t)^{1/2}(\ubb_0-\ubb)\|_2\sqrt{\x_i^\top Z(t)^{-1}\x_i}\\
&\leqslant\sqrt{\beta_t\x_i^\top Z(t)^{-1}\x_i}.
\end{align*}

For the $\mathcal{C}_t^1$ case:
\begin{align*}
|(\ubb_0-\ubb)^\top\x_i|&\leqslant \|Z(t)^{1/2}(\ubb_0-\ubb)\|_1\cdot\|Z(t)^{-1/2}\x_i\|_\infty\\
&\leqslant \|Z(t)^{1/2}(\ubb_0-\ubb)\|_1\cdot\|Z(t)^{-1/2}\x_i\|_2\\
&\leqslant\sqrt{nd\beta_t\x_i^\top Z(t)^{-1}\x_i}.
\end{align*}

Denote $w_i(t)=\sqrt{\x_i(t)^\top Z(t)^{-1}\x_i(t)}$ and $W(t)=\sum_i w_i(t)$, the ``width'' in the direction of the chosen $\x$. 
We now study $\rho(t)=\ubb_0\top L(t)\vbb(t)-\ubb_0^\top L(t)\vbb_*(t)$, the instantaneous regret acquired on round $t$.
If $\tilde{\ubb}$ is the choice maximizing $\tilde{\ubb}^\top L(t) \vbb(t)$ and knowing that $\tilde{\ubb}_t^\top L(t)\vbb(t)\leqslant \ubb^\top L(t)\vbb^*$ we have: 
\begin{align*}
  \rho(t)&=\sum_i \ubb_0^\top\x_i(t)-\ubb_0^\top\x_{i*}(t)
  \leqslant\sum_i (\ubb_0-\tilde{\ubb})^\top \x_i(t)
  \leqslant 2\sqrt{\beta_t}W(t).
\end{align*}

We now study the evolution of $Z_t$:
\begin{align*}
  Z(t+1)&=Z(t)+\sum_{i=1}^n \x_i(t)\x_i(t)^\top\\
  &=Z(t)^{1/2}\left(I+\sum_i Z(t)^{-1/2}\x_i(t)\x_i(t)^\top Z(t)^{-1/2}\right)Z(t)^{1/2}\\
  &=Z(t)\left(I+\sum_i w_i(t)^2I\right),
\end{align*}
and so the determinant is:
\[
\det Z(t+1)=\left(1+\sum_i w_i(t)^2\right)\det Z(t).
\]

To prove this, we observe that the same arguments as in~\cite{dani2008stochastic,abbasi2011improved} can be used in our case.

If we take the trace:
\begin{align*}
\trace Z(t)=\trace\left(I+\sum_{\tau<t}\sum_{i=1}^n \x_i(t)\x_i(t)^\top\right)\leqslant nd+tn^2,
\end{align*}
Here, we assume that the first $nd$ rows form the standard basis) and, again, use the same argument as in~\cite{dani2008stochastic} ($\det Z_t$ equals the product of the eigenvalues, $\trace Z_t$ equals the sum of the eigenvalues, so $\det Z_t$ is maximized when eigenvalues are all equal) and using the inequality of arithmetic and geometric means, then:
\[
  \det Z(t)\leqslant \left(\frac{\trace Z(t)}{nd}\right)^{nd}\leqslant\left(1+\frac{n}{d}t\right)^{nd}.
\]

This allows us to derive the following bound (by Cauchy-Schwarz, and using the fact that $\log(1+y)\geqslant \frac{y}{1+y}$ and that $\sum_i w_i(t)^2\leqslant n$):
\begin{align*}
 \sum_{\tau=1}^{t}W(t)^2&\leqslant  n\sum_{\tau=1}^{t}\min\left(\sum_i        w_i(t)^2,n\right)
  \leqslant 2n^2\sum_{\tau=1}^{t}\log\left(1+\sum_i w_i(t)^2\right)\\
  &=2n^2\log(\det Z(t))
  \leqslant 2n^3d\log \left(1+\frac{n}{d}t\right).
\end{align*}

Bringing it all together:
\begin{align*}
  \sum_{t=1}^T \rho(t)^2&\leqslant \sum_{t=1}^T4\beta_t\min(w_t^2,1)
  \leqslant 4\beta_t\sum_{t=1}^T\min(w_t^2,1)\\
  &\leqslant 8\beta_tn^3d\log \log \left(1+\frac{n}{d}T\right).
\end{align*}

Using the Cauchy-Schwarz inequality again, we have for $\mathcal{C}_t^2$:
\begin{align*}
    \sum_{t=1}^T\rho(t) &\leqslant \sqrt{T\sum_{t=1}^T \rho(t)^2}
        \leqslant\sqrt{8n^3d\beta_TT\ln  \left(1+\frac{n}{d}T\right)},
\end{align*}
and, using the same steps for $\mathcal{C}_t^1$:
\begin{align*}
\sum_{t=1}^T\rho(t) &\leqslant nd\sqrt{8n^2\beta_TT\ln \log \left(1+\frac{n}{d}T\right)}.
\end{align*}

This concludes the proof.
\end{proof}

\begin{lemma}
  Given $\delta>0$, $\forall t$, $\text{Pr}(\ubb_0\in\mathcal{C}_t^2)\geqslant 1-\delta.$
\end{lemma}

\begin{proof}\emph{(Sketch)}
  To prove this, we need to analyze the measure 
  \[
  E_t=\left(\ubb_0-\hat{\
  \ubb}(t)\right)^\top Z(t)\left(\ubb_0-\hat{\ubb}(t)\right)
\]
 and how it grows with $t$, and then try to bound it.
  
  The proof is the same as the one in~\cite{dani2008stochastic}, with the following change: instead of analyzing the change at step $t$, we analyze it on as user-by-user basis. In other words, instead of $T$ timesteps, we have $nT$ timesteps.
  
  The changes from the proof in~\cite{dani2008stochastic} regard the definitions, at a given timestep $\tau$:
   $   Z(\tau+1)=Z(\tau)+\x(\tau)\x(\tau)^\top,$
      $Y(\tau+1):=Z(\tau+1)\left(\ubb_0-\hat{\ubb}(\tau)\right)=Y(\tau)+\epsilon(\tau)\x(\tau)\x(\tau)^\top,$
  which allows us to write the growth of $E(\tau)$:
   $ E(\tau+1)=Y(\tau+1)^\top Z(\tau+1)^{-1}Y(\tau+1)
    =Y(\tau)^\top Z(\tau+1)^{-1}Y(\tau)+2\epsilon(\tau)\x(\tau)^\top Z(\tau+1)^{-1}Y(t)+\epsilon(\tau)^2\x(\tau)^\top Z(\tau+1)^{-1}\x(\tau).$
  
  Using the matrix inversion lemma, and the fact that $E(0)\leqslant nd$, we get the following bound on $E(t)$:
  \begin{align*}
        E(t)&\leqslant E(0)+2\sum_{\tau<t}\epsilon(\tau)\frac{\x(\tau)^\top(\ubb_0-\ubb(\tau))}{1+w(\tau)}+\sum_{\tau<t}\epsilon(\tau)^2\frac{w(\tau)^2}{1+w(\tau)^2}
  \end{align*}
  
  To bound this, we use the martingale difference sequence defined as:
  \[
    \mathcal{M}(\tau):=2\mathcal{E}(\tau)\frac{\x(\tau)^\top(\ubb_0-\hat{\ubb}(\tau))}{1+w(\tau)},
  \]
  where $\mathcal{E}(\tau)$ is an escape event encoding whether $E(\tau)$ is in the ball at timestep $\tau$:
  \[
    \mathcal{E}(\tau):=\mathbf{I}\{E(\tau)\leqslant\beta_\tau,\forall\tau\}.
  \]
  
  We will apply Freedman's theorem like in~\cite{dani2008stochastic}, by computing the variance:
  \begin{align*}
  \text{Var}(\tau)&:=\text{Var}(\mathcal{M}(\tau)\mid\mathcal{M}(1),\dots,\mathcal{M}(\tau-1))=8\beta_\tau nd\ln{\tau},
  \end{align*}
  and choosing the constants properly:
 \begin{align*}
      &\text{Pr}\left(\sum_t\mathcal{M}(t)\geqslant \beta_\tau/2\right)=\\
      &\text{Pr}\left(\sum_t\mathcal{M}(t)\geqslant~\text{and}~\text{Var}(\tau)\leqslant 8nd\beta_t\ln t\right)\\
      &\leqslant\max\left\{\exp\left(\frac{-\beta_\tau}{128nd\ln \tau}\right),\exp\left(\frac{-3\sqrt{\beta_\tau}}{8}\right)\right\}
      \leqslant\frac{\delta}{\tau^2}.
  \end{align*}
  Using the union bound, we get that:
  \begin{align*}
      \text{Pr}\left(\sum_{t<\tau}\mathcal{M}(t)\geqslant \beta_\tau/2,\forall\tau\right)\leqslant\delta.
  \end{align*}
  Using the above results, allows us to derive, with probability $1-\delta$, that:
  \begin{align*}
      E(t)\leqslant nd+2nd\ln t+\beta_t/2
      \leqslant\beta_t,
  \end{align*}
  which is the desired result.
\end{proof}
The proof for the $\mathcal{C}_t^1$ case is analogous, due to norm equivalence. Note that the confidence bounds and $\beta_t$ remain the same, due to the $\sqrt{nd}$ factor added in the definition of $\mathcal{C}_t^1$.

\section{Proof of Theorem~\ref{th:ts_regret}}\label{sec:ts_proof}

For \textsc{ThompsonSampling}, we analyze the Bayesian regret:
\begin{align*}
  \text{BR}(T)&= \mathbf{E}\left[\sum_{t=1}^T\left(\ubb_0^\top L(t)\vbb_*(t)-\ubb_0^\top L(t)\vbb(t)\right)\right]\\
  &=\sum_{t=1}^T\sum_{i=1}^n\mathbf{E}\left[\ubb_0^\top \x_{i*}(t)-\ubb_0^\top \x_i(t)\right].
\end{align*}
We focus on the case that when $\ubb$ is known to be in the confidence ball defined by $\beta_T$. The formula of $\beta_T$ is a straight adaptation of Thm.\ 20.5 of~\cite{lattimore2019bandit} (confidence bounds for linear estimators) and is not repeated here. The following is an adaptation of the proof of Thm.\ 36.4 of the same book. 

Denote as $U_{t}(\x_i)=\hat{\ubb}_0^\top\x_i(t)+\beta_tw_i(t)$ and $\mathcal{E}_i(t)$ the event that $\|\hat{\ubb}_0(t)-\ubb_0\|\leqslant\beta_t$, $\bar{\mathcal{E}_i}(t)$ the negated event, $\mathcal{E}(t)=\cap_{[n]} \mathcal{E}_i(t)$, and $\mathcal{E}=\cap_{[t]} \mathcal{E}(t)$.
Then:
\begin{align*}
    \text{BR}(T)&=\mathbf{E}\left[\sum_{t=1}^T\left(\ubb_0^\top L(t)\vbb_*(t)-\ubb_0^\top L(t)\vbb(t)\right)\right]\displaybreak[0]\\
    &=\mathbf{E}\left[\mathbf{I}_{\bar{\mathcal{E}}}\sum_{t=1}^T\left(\ubb_0^\top L(t)\vbb_*(t)-\ubb_0^\top L(t)\vbb(t)\right)\right]\displaybreak[0]\\
    &~~+\mathbf{E}\left[\mathbf{I}_\mathcal{E}\sum_{t=1}^T\left(\ubb_0^\top L(t)\vbb_*(t)-\ubb_0^\top L(t)\vbb(t)\right)\right],  
\end{align*}
and since $\exists t\leqslant T.\|\hat{\ubb}(t)-\ubb_0\|\leqslant 1/T$:
\begin{align*}
\text{BR}(T)&\leqslant 2+\mathbf{E}\left[\sum_{t=1}^T\sum_{i=1}^n \mathbf{I}_{\mathcal{E}_{i}(t)}\left(\ubb_0^T\left(\x_{i*}(t)-\x_i(t)\right)\right)\right].
\end{align*}

We need to bound:
\begin{align*}
  &\mathbf{E}_t\left[\sum_{i=1}^n \mathbf{I}_{\mathcal{E}_{i}(t)}\left(\ubb_0^T\left(\x_{i*}(t)-\x_i(t)\right)\right)\right]\\
  &\leqslant\sum_{i=1}^{n}\mathbf{I}_{\mathcal{E}_{it}}\mathbf{E}_{it}\left[\ubb_0^\top\x_{i*}(t)-U_t(\x_{i*})+U_t(\x_i)-\ubb_0^\top\x_i(t)\right],
\intertext{(because $\mathbf{E}_{it}\left[U_t(\x_{i*})\right]=\mathbf{E}_{it}\left[U_t(\x_{i})\right]$)}
&\leqslant\sum_{i=1}^{n}\mathbf{I}_{\mathcal{E}_{it}}\mathbf{E}_{it}\left[U_t(\x_i)-\ubb_0^\top\x_i(t)\right]\\
&\leqslant \sum_{i=1}^{n}\mathbf{I}_{\mathcal{E}_{it}}\mathbf{E}_{it}\left[\left(\hat{\ubb}_0(t)-\ubb_0\right)^\top\x_i(t)\right]+\beta_tw_i(t)\\
&\leqslant 2\beta_t\sum_i^nw_i(t)=2\beta_tW(t).
\end{align*}
Then:
\begin{align*}
  \mathbf{E}&\left[\sum_{t=1}^T\sum_{i=1}^n \mathbf{I}_{\mathcal{E}_{i}(t)}\left(\ubb_0^T\left(\x_{i*}(t)-\x_i(t)\right)\right)\right]
  \leqslant 2\beta_T\sum_{t=1}^T W(t)\\
  &\leqslant 2\beta_T\sqrt{T\sum_{t=1}^T W^2(t)}\leqslant 2\beta_T\sqrt{2n^3d\log\left(1+\frac{n}{d}T\right)}.
\end{align*}
The regret bound follows.

%% file: sdp.tex
\section{SDP Relaxation for LinUCB}\label{app:linucb}

We use the same steps as in~\cite{lu2014optimal}, and we assume that the optimization is made on the $L_2$ ball for each $v_i\in \mathbb{R}^d$, $i\in[n]$, so the optimization becomes:
\begin{align*}
    \textbf{max}:&\quad \left( L(t)^\top\mathbf{\hat{u}}_{t}\right)^\top \mathbf{v}  + \mathbf{v}^\top cL(t)^\top \Sigma_{t} L(t) \mathbf{v} \\
    \text{s.t.}:&\quad \mathbf{v}^2\in \mathcal{D}_0,
\end{align*}
where $\mathcal{D}_0=\left\{ \mathbf{v}'\in\mathbb{R}^{nd}| \forall{i\in[n]},\sum_{j=1}^{nd}\mathbb{I}_{\lceil{j/n}\rceil=i}~ \mathbf{v}'_i\leqslant1\right\}$.
Removing the linear term by writing $\mathbf{y}=(\mathbf{v},t)\in\mathbb{R}^{nd+1}$, $t\leqslant 1$, we obtain:
\begin{align*}
    \textbf{max}:&\quad \mathbf{y}^\top H_0 \mathbf{y}\\
    \text{s.t.}:&\quad \mathbf{y}^2\in \mathcal{D},
\end{align*}
where:
\[
    H_0=
    \begin{bmatrix}
        cL(t)^\top \Sigma_{t} L(t) & L(t)^\top\mathbf{\hat{u}}_{t}/2\\
        \mathbf{\hat{u}}_{t}^\top L(t)/2 & 0
    \end{bmatrix}\in\mathbb{R}^{(nd+1)\times(nd+1)},
\]
and: $\mathcal{D}=\left\{(\mathbf{y}',t')\in\mathbb{R}^{nd+1}|\mathbf{y}'\in\mathcal{D}_0,t'\leqslant1\right\}.$
The SDP relaxation is made by setting $Y=\mathbf{y}\mathbf{y}^\top$ and solving the following optimization:
\begin{align}
    \textbf{max}:&\quad \text{tr}(HY)\\
    \text{s.t.}:&\quad Y\succeq 0\nonumber\\
    &\quad \text{diag}(Y)\in\mathcal{D}.\nonumber
\end{align}